\documentclass[11pt]{article}
\usepackage{graphicx}
\usepackage[utf8]{inputenc}
\usepackage[T1]{fontenc} 
\usepackage[english]{babel}
\usepackage{verbatim}
\usepackage{calligra}
\usepackage{enumerate}
\usepackage{enumitem} 
\usepackage{dsfont}
\usepackage{amsfonts}
\usepackage{amsmath}
\usepackage{amsthm}
\usepackage{amssymb}
\usepackage{mathtools}
\usepackage{etoolbox}
\usepackage{mathrsfs}
\usepackage{color}
\usepackage{graphicx}
\usepackage{bm}
  \usepackage[all]{xy}
\usepackage{hyperref}
\usepackage{comment}
\hypersetup{
    colorlinks=true,
    linkcolor=black,
    filecolor=black,      
    urlcolor=black,
}

\usepackage{tikz}
\usetikzlibrary{arrows.meta}

\usepackage{graphicx}		
\usepackage[hypcap=false]{caption}

\DeclareMathOperator{\Var}{Var}
\DeclareMathOperator{\Cov}{Cov}

\theoremstyle{plain}
\newtheorem{defi}{Definition}
\newtheorem{thm}[defi]{Theorem}
\newtheorem{prop}[defi]{Proposition}
\newtheorem{coro}[defi]{Corollary}
\newtheorem{lemma}[defi]{Lemma}

\theoremstyle{definition}
\newtheorem{ex}[defi]{Example}
\newtheorem{rmk}[defi]{Remark}

\AtEndEnvironment{ex}{\hfill\ensuremath{\diamond}}
\AtEndEnvironment{rmk}{\hfill\ensuremath{\diamond}}

\newcommand{\mbb}{\mathbb}
\newcommand{\dst}{\displaystyle}

\newcommand{\msf}{\mathsf}
\newcommand{\mc}{\mathcal}

\newcommand{\msr}{\mathscr}
\newcommand{\ve}{\varepsilon}

\newcommand{\de}{\delta}

\newcommand{\R}{\mathbb{R}}

\newcommand{\N}{\mathbb{N}}

\renewcommand{\P}{\mathbb{P}}

\newcommand{\Ker}{{\rm Ker}\,}

\newcommand{\di}{{\rm d}}
\newcommand{\dx}{ \, {\rm d} x}

\allowdisplaybreaks

\usepackage[a4paper,margin=2.5cm]{geometry}

\begin{document}

\newcommand{\ad}[1]{\textcolor{teal}{[Arnaud: #1 ]}}
\newcommand{\gl}[1]{\textcolor{blue}{[Geoffrey: #1]}}

\title{\textsc{\Large{\textbf{Fluctuations of Nonlinear Observables in Mean Field Neural Network Training}}}}

\author{Arnaud Descours\thanks{Université Lyon 1, LSAF, EA 2429, Lyon, France, \texttt{arnaud.descours@univ-lyon1.fr}}
\qquad 
 Geoffrey Lacour\thanks{Université Paris-Saclay, INRAE, MaIAGE, UR 1404, 78350 Jouy-en-Josas, France, \texttt{geoffrey.lacour@inrae.fr}}}

\vspace{.2cm}
\date{}

\maketitle

\begin{abstract}
Mean field limits describe the training dynamics of wide neural networks through the evolution of the empirical distribution of their parameters. Although functional central limit theorems characterize the asymptotic fluctuations of this distribution, quantities of practical interest are typically nonlinear observables of the parameter distribution rather than the distribution itself. In this work, we show how these mean field fluctuations propagate to finite dimensional nonlinear observables for shallow neural networks trained by stochastic gradient descent. Working in the weighted Sobolev space in which the limiting fluctuation process is constructed, we apply a functional Delta method under ordinary Fréchet differentiability, without requiring Lions derivatives with respect to the measure variable. We obtain a central limit theorem for the observables and, under a suitable representation of their differentials, an explicit covariance formula inherited from the underlying mean field fluctuation theory.

We also study whether prescribed quantities of interest can be recovered from the selected observations. Under a constant rank assumption, we prove that a quantity of interest factors locally through the observation functional if and only if, throughout a neighborhood, the kernel of the differential of the observation is contained in that of the quantity of interest. Thus, a differential condition expressed directly in the ambient Sobolev space yields an exact nonlinear local factorization. These results provide a framework both for quantifying finite-width uncertainty on observable, statistically or physically meaningful quantities and for assessing whether the chosen observations contain the information required to identify them.
\end{abstract}


\smallskip

{\footnotesize \textbf{Mathematics Subject Classification (MSC2020):} 68T07, 60F17, 60B12.
}

\section{Introduction}

In this work, we consider two-layer (or shallow) neural networks, namely neural networks consisting of a single hidden layer. Such models provide one of the simplest architectures for which the mean-field limit can be analyzed rigorously while already capturing many of the fundamental mechanisms underlying modern overparameterized learning. Given an input datum $x \in \R^m$, the network prediction is obtained by combining the contributions of the neurons in the hidden layer. Each neuron is characterized by a parameter vector $w \in \R^d$, where $d \in (\N\backslash\{0\})$ denotes the dimension of the parameter space associated with a single neuron. During training, these parameters are updated through stochastic gradient descent (SGD) in order to minimize a prescribed loss function. Beginning with the seminal works by F.Bach \& L.Chizat~\cite{bach-chizat-2023}, by S.Mei, A.Montanari \& P.-M.Nguyen~\cite{mei-montanari-nguyen-2018} and by J.Sirignano \& K.Spiliopoulos~\cite{sirignano-spiliopoulos-2020}, a large body of literature has shown that, rather than tracking each parameter individually, it is more convenient in the large-width regime to describe the state of the network through the empirical distribution of the neuron parameters. This point of view provides a rigorous mean-field description of SGD through the evolution of a probability measure and has led to a detailed understanding of both the deterministic training dynamics and their asymptotic fluctuations.

More precisely, let $\mu_t^N$ denote the empirical distribution of the network parameters at step $t$. Under suitable assumptions, the sequence $((\mu_t^N)_{N \geq 1}$ converges, as the width $N$ of the hidden layer tends to infinity, towards a deterministic measure-valued process $(\bar\mu_t)_{t \geq 0}$ solving a nonlinear continuity equation. Beyond this law of large numbers, the fluctuations around the mean-field limit,
\[
	\eta_t^N := \sqrt N\left(\mu_t^N - \bar\mu_t\right),
\]
satisfy a functional central limit theorem. A further analysis allows one to characterize explicitly the law of the limiting fluctuation process, yielding a complete probabilistic description of the uncertainty induced by finite-width effects. Since these results constitute the mathematical foundation of the present work, we review the mean-field framework in Section~\ref{sec:mean-field-dynamics} together with the corresponding law of large numbers and central limit theorem in Section~\ref{sec:preliminary-results}. In particular, the explicit characterization of the limiting Gaussian process established in the work by A.Descours, A.Guillin, G.Lacour, M.Michel, B.Nectoux \& P.Stos~\cite[Theorem 3]{descours-guillin-lacour-michel-nectoux-stos-2026} will play a central role throughout the paper.

While these results provide a precise description of the parameter dynamics, they do not directly address the quantities that are actually observed in practical applications. Indeed, the empirical measure itself is not an observable of primary interest. Instead, one typically monitors quantities that are computed from it, such as the network prediction associated with a prescribed input, the empirical loss, or other statistics characterizing the state of the network. From the mean-field perspective, all these quantities can naturally be viewed as functionals acting on the empirical distribution of the parameters. This observation provides the starting point of the present work.

This motivates the introduction of the notion of an \emph{observation functional},
\[
	\Phi:\mc M(\R^d) \to \R^k,
\]
which associates with the empirical distribution of the network parameters a quantity of practical interest. Here, $\mc M(\R^d)$ denotes the space of finite Borel measures on the parameter space $\R^d$, while $k \in (\N\backslash\{0\})$ is the dimension of the observed quantity.

A natural question is therefore whether the uncertainty described by the mean-field fluctuation theory propagates to such observation functionals. Since the asymptotic fluctuations of the empirical measure are completely characterized by the functional central limit theorem recalled in Section~\ref{sec:preliminary-results}, one may ask whether an analogous central limit theorem holds for the observed quantities themselves. 

Although observation functionals are naturally introduced as mappings acting on the space of measures, the fluctuation process takes values in the dual of a weighted Sobolev space of the form
\[
	\msr H := \mc H^{-J_0+1,j_0}(\R^d),
\]
for some parameters $J_0, j_0$ recalled in the statement of Theorem~\ref{thm:functional-delta-method} below and whose structure is detailed in Section~\ref{ssec:functional-framework}. Although Wasserstein spaces and
duals of weighted Sobolev spaces carry different structures, probability measures satisfying suitable moment conditions can be canonically viewed in both settings. The latter additionally provide a linear Hilbert space framework, in which statistical, probabilistic and analytical tools take a particularly natural form. 

More broadly, methods base on Sobolev spaces have been especially fruitful in the analysis of inverse problems for partial differential equations arising from mean field models. Recent examples include the works by O.Imanuvilov, H.Liu, \& M.Yamamoto~\cite{imanuvilov-liu-yamamoto-2024} for state determination and inverse source problems in mean field game systems , as well as the inverse coefficient problem studied by O.Imanuvilov \& M.Yamamoto~\cite{imanuvilov-yamamoto-2025}. Mean field game systems themselves arise as continuum limits of Nash equilibrium problems involving a large number of interacting players, as established in the works by J.-M.Lasry \& P.-L.Lions \cite{lasry-lions-2006-I,lasry-lions-2006-II,lasry-lions-2007}.

Accordingly, throughout the theoretical analysis, observation functionals are understood as mappings defined on this dual space, which contains the empirical measures through the canonical embedding recalled in subsection~\ref{ssec:functional-framework}. This leads to our first main result below.

\begin{thm}[Propagation of mean-field fluctuations to observation functionals]\label{thm:functional-delta-method} Assume that Assumptions~\ref{ass:data}--\ref{ass:batch-limit} hold and that $\beta > \frac{3}{4}$. Let
\[
    J_0 \geq 4\bigg\lceil\frac{d}{2}\bigg\rceil + 8, \quad j_0 = \bigg\lceil\frac{d}{2}\bigg\rceil + 2,
\]
set $\msr H := \mc H^{-J_0+1,j_0}(\R^d)$ and $\Phi: \msr H \to \R^k$ be a Fréchet differentiable observation functional. Then, for every $t \geq 0$,
\[
	\sqrt N\Big(\Phi(\mu_t^N) - \Phi(\bar\mu_t)\Big)\xrightarrow[N\to\infty]{\mathscr L} D\Phi(\bar\mu_t)\bar\eta_t.
\]
In particular, the limiting fluctuations of the observation functional are Gaussian. 

Moreover, writing $\Phi=(\Phi_1,\dots,\Phi_k)$ and $Z_t^\Phi:=D\Phi(\bar\mu_t)\bar\eta_t$, the covariance matrix $\Sigma_t^\Phi$ of $Z_t^\Phi$ is given by
\[
    \Sigma_t^\Phi = \Big(\Sigma_{i,j}^\Phi(t)\Big)_{1\leq i,j \leq k},
\]
where
\[
    \Sigma_{i,j}^\Phi(t) = \Cov\Big( D\Phi_i(\bar\mu_t)(\bar\eta_t), D\Phi_j(\bar\mu_t)(\bar\eta_t)\Big).
\]
\end{thm}

The proof relies on a functional Delta method, which provides a general mechanism for transferring central limit theorems through differentiable mappings. Such techniques are classical in asymptotic statistics and have found numerous applications in infinite-dimensional settings, see for instance the works by E.Beutner~\cite{beutner-2026}, by E.Beutner \& H.Z{\"a}hle~\cite{beutner-zahle-2010,beutner-zahle-2016}, by J.Cupidon, D.Gilliam, R.Eubank \& F.Ruymgaart~\cite{cupidon-gilliam-eubank-ruymgaart-2007}, by R.Schmidt \& U.Stadtm{\"u}ller~\cite{schmidt-stadtmuller-2006} or the books by P.Andersen, {\O}.Borgan,  R.Gill \& N.Keiding~\cite[Section II.8]{andersen-borgan-gill-keiding-1995}, by M.Kosorok~\cite[Section 2.2.4. and Chapter 12]{kosorok-2008} and A.Van der Vaart~\cite[Chapter 20]{vandervaart-2000}. Their application to mean-field models is, however, often considerably more delicate. Indeed, convergence towards the limiting fluctuation process is typically established through compactness arguments in spaces of probability measures, while differentiability with respect to the measure variable is commonly formulated in terms of Lions derivatives (or $L$-derivatives). We refer for example to the works by R.Carmona \& F.Delarue~\cite{carmona-delarue-2018} or by B.Jourdain \& A.Tse~\cite{jourdain-tse-2021}. Although this framework is remarkably powerful, it generally requires rather strong regularity assumptions on the considered functionals.

The present setting turns out to be substantially simpler. Owing to the weighted Sobolev framework underlying the fluctuation theory, the empirical measures and their limiting fluctuations naturally evolve in a Banach space in which standard Fréchet differentiability is sufficient. As a consequence, the functional Delta method can be implemented under ordinary Fréchet differentiability, without requiring a measure specific differential calculus, while preserving an explicit Gaussian asymptotic description.

\begin{coro}\label{coro:covariance-formula} Assume that the assumptions of Theorem~\ref{thm:functional-delta-method} hold. If for every $i\in\{1, \ldots, k\}$, there exists $\varphi_i^t\in C^\infty_b(\R^d)$ such that
\[
    D\Phi_i(\bar\mu_t)(h) = \langle \varphi_i^t,h\rangle, \quad h \in \msr H,
\]
then the covariance matrix $\Sigma_t^\Phi$ of $Z_t$ is given by 
\[
	\Sigma_{i,j}^\Phi(t) = \Cov\Big(f^{\varphi_i^t}(0,W_0^1), f^{\varphi_j^t}(0,W_0^1)\Big) + \alpha^2\mbb E\bigg[\frac{1}{|B_\infty|}\bigg]\int_0^t\Cov_\pi\Big(Q_r[f^{\varphi_i^t}(r)], Q_r[f^{\varphi_j^t}(r)]\Big)\,dr .
\]
Here, $f^{\varphi_i^t}$ denotes the solution of the backward transport equation \eqref{eq:backward-transport} associated with the terminal datum $\varphi_i^t$, as in Theorem~\ref{thm:mean-field-uq}.
\end{coro}

The previous result explains how finite-width uncertainty propagates from the empirical parameter distribution to a prescribed observation functional. This raises a complementary question. Given a vector-valued quantity of interest
\[
    Q: \mc U \subset \msr  H \to \R^\ell,
\]
does the observation $\Phi(\mu)$ contain enough information to determine
$Q(\mu)$, at least locally? We say that $Q$ is locally identifiable from $\Phi$ near $\mu_0$ in the differentiable sense if there exists a continuously differentiable mapping $F$, defined near $\Phi(\mu_0)$, such that
\[
    Q=F\circ\Phi
\]
in a neighborhood of $\mu_0$. A necessary differential condition for such a factorization is
\[
    \Ker D\Phi(\mu) \subset \Ker DQ(\mu)
\]
which means that every infinitesimal perturbation that is invisible to the observation must also leave the quantity of interest unchanged. Our second main
result shows that, provided $D\Phi$ has constant finite rank near $\mu_0$, the validity of this kernel inclusion throughout a neighborhood is also sufficient for the existence of the local nonlinear factorization having the form $Q = F\circ\Phi$.

This viewpoint is closely related to classical local identifiability criteria for finite-dimensional parametric models, which are commonly formulated in terms of rank conditions on the Jacobian of the observation map. We refer, for instance to the works by X.Chen, V.Chernozhukov, S.Lee \& W.Newey~\cite{chen-chernozhukov-lee-newey-2014} or by  T.Rothenberg~\cite{rothenberg-1971} and the subsequent literature on parametric identifiability as in the book by E.Walter~\cite{walter-1987} and references therein.  Identifiability questions also arise naturally in the inference of interacting particle systems, where the interaction laws governing the dynamics must be recovered from trajectory data. In this direction, F.Lu, M.Maggioni \& S.Tang~\cite{lu-maggioni-tang-2022} developed nonparametric learning frameworks for interaction kernels in stochastic particle systems, while Z.Li, F.Lu, M.Maggioni, S.Tang \& C.Zhang~\cite{li-lu-maggioni-tang-zhang-2021} established coercivity conditions ensuring identifiability and clarified their role in the mean field regime. At the level of the limiting equations, Q.Lang \& F.Lu~\cite{lang-lu-2023} further characterized the function spaces on which interaction kernels are identifiable from the corresponding mean field dynamics. These contributions illustrate the importance of identifiability questions for systems of interacting particles, both before and after passage to the mean field limit.

In particular, when the quantity of interest coincides with the full parameter vector, our criterion reduces in finite dimension to the usual full rank condition. More generally, it provides an infinite-dimensional local identifiability criterion for prescribed quantities of interest, even when the underlying parameter distribution itself is not identifiable.

\begin{thm}[Differentiable local identifiability criterion]\label{thm:local-identifiability}
Let $\mc U \subset \msr H$ be open, and assume that $\Phi: \mc U \to \R^k$, $Q:\mc U \to \R^\ell$ are continuously Fréchet differentiable mappings. Let $\mu_0 \in \mc U$, and assume that $D\Phi$ has constant rank in a neighborhood of $\mu_0$, namely that there exist an open neighborhood $\mc U_0 \subset \mc U$ of $\mu_0$ and $r\in\{0,\ldots,k\}$ such that for every $\mu \in \mc U_0$, it holds $\mathrm{rk}\, D\Phi(\mu) = r$. Then the following assertions are equivalent.

\begin{enumerate}[label=(\roman*)]
    \item There exist open neighborhoods $\mc U_1 \subset \mc U_0$ of $\mu_0$ and $\mc V \subset \R^k$ of $\Phi(\mu_0)$, together with a continuously differentiable mapping $F: \mc V \to \R^\ell$ such that $\Phi(\mc U_1) \subset \mc V$ and for every $\mu \in \mc U_1$
    \[
        Q(\mu) = F(\Phi(\mu)).
    \]

    \item There exists an open neighborhood $\mc U_2 \subset \mc U_0$ of $\mu_0$ such that for every $\mu \in \mc U_2$
    \[
        \mathrm{Ker}\, D\Phi(\mu) \subset \mathrm{Ker}\, DQ(\mu).
    \]
\end{enumerate}
\end{thm}

The interpretation of Theorem~\ref{thm:local-identifiability} is as follows. Assertion~(i) states that the quantity of interest $Q$ can be reconstructed, in a neighborhood of $\mu_0$, by applying a continuously differentiable mapping to the observation $\Phi(\mu)$, so that $Q$ is locally identifiable from $\Phi$ at $\mu_0$ in a differentiable sense. In particular, under either of the equivalent assertions, there exists a neighborhood $\mc U_3$ of $\mu_0$ such that, for every $\mu,\nu\in \mc U_3$,
\[
    \Phi(\mu) = \Phi(\nu) \Rightarrow Q(\mu) = Q(\nu).
\]
Thus, two nearby elements that are indistinguishable through the observation functional $\Phi$ are also indistinguishable with respect to the quantities of interest encoded by $Q$.

The geometric content of the result can be understood in terms of the local fibers of the observation functional. The constant rank assumption implies that, near $\mu_0$, the level sets $\Phi^{-1}(y)$ are continuously differentiable submanifolds of $\msr H$, whose tangent spaces are given by
\[
    T_\mu\mc F_\mu = \mathrm{Ker}\, D\Phi(\mu),
\]
where we set the fiber $\mc F_\mu := \{\nu \in \mc U_0 \,\mid\, \Phi(\nu) = \Phi(\mu)\}$. The kernel inclusion given in Assertion~(ii) therefore means that the differential of $Q$ vanishes along every direction tangent to the local fibers of $\Phi$. Consequently, $Q$ is constant along each sufficiently small connected piece of such a fiber. It can thus be regarded locally as a function of the observation $\Phi(\mu)$ alone, which yields the desired factorization. 

The theorem also provides a structural preliminary test for inverse problems. Before attempting to reconstruct $Q(\mu)$ from $\Phi(\mu)$, one must verify that the selected observations do not discard directions along which the quantities of interest vary. The kernel criterion can therefore guide the design of informative observation functionals. Questions of stability, conditioning, regularization, and numerical reconstruction lie beyond the scope of the present
work. 

Typical observation functionals include network outputs evaluated at prescribed inputs, losses, and statistics of the parameter distribution, whereas quantities of interest may represent moments, predictive summaries, or effective physical parameters. The following example illustrates how our fluctuation result can be applied to an effective thermal diffusivity coefficient reconstructed from a trained neural-network surrogate.

\begin{ex}\label{ex:temperature} 

Let us consider a finite dataset $\big\{(z_k,y_k)\big\}_{k=1}^n := \big\{((t_k,x_k),y_k)\big\}_{k=1}^n \subset \mc X \times \mc Y$, where $\mc X = [0,T]\times \Omega$, with $T>0$ and a bounded domain $\Omega\subset\R^{m-1}$, $m \geq 2$, and where $\mc Y = \R$. The observations $y_k = \theta(t_k,x_k)$ represent temperatures measured at spatial positions $x_k$ and times $t_k$. 

Our aim is to estimate the effective thermal diffusivity of the material, viewed here as a quantity of interest $Q(\mu)$. We therefore seek a suitable observation functional $\Phi(\mu^N)$ from which this coefficient can be recovered. The neural network $\theta_\mu:=\langle\sigma^*,\mu\rangle$ may be trained by SGD as a surrogate for the temperature field through the associated nonlinear regression problem.

Define the open subset $\mc U := \big\{\mu\in \msr H\, \mid \, \Delta_x\theta_\mu \neq 0 \; \text{in}\, L^2(\mc X)\big\}$.  For every $\mu \in \mc U$, we may define the effective diffusivity as the least-squares coefficient
\[
	Q(\mu) := \underset{\lambda \in \R}{\mathrm{argmin}}\|\partial_t\theta_\mu - \lambda\Delta_x\theta_\mu\|^2_{L^2(\mc X)},
\] 
so that it is uniquely given by the formula
\[
	Q(\mu) := \frac{(\partial_t\theta_\mu, \Delta_x\theta_\mu)_{L^2(\mc X)}}{\|\Delta_x \theta_\mu\|^2_{L^2(\mc X)}}.
\]
Since the derivatives of the network can be evaluated by automatic differentiation, they provide natural observation functionals. Also, under the standing regularity assumptions on the activation function, the mappings $v \mapsto \partial_t\theta_v$ and $v \mapsto\Delta_x\theta_v$ are continuous from $\msr H$ into $L^2(\mc X)$. More precisely, for every $z \in \mc X$,
\[
    |\partial_t\theta_v(z)| \leq \|v\|_{\msr H}\|\partial_t\sigma^*(\cdot,z)\|_{\msr H'},
\]
and similarly
\[
    |\Delta_x\theta_v(z)|  \leq \|v\|_{\msr H}\|\Delta_x\sigma^*(\cdot,z)\|_{\msr H'}.
\]
Thanks to assumption~\ref{ass:sigma} (see Section~\ref{ssec:assumptions}) on $\sigma$, together with the finite measure of $\mc X$, then yield
\[
    \|\partial_t\theta_v\|_{L^2(\mc X)} + \|\Delta_x\theta_v\|_{L^2(\mc X)}\leq C\|v\|_{\msr H}.
\]

Also, note that the above continuity of $\mu \mapsto \Delta_x\theta_\mu$ from $\msr H$ into $L^2(\mc X)$ shows that $\mc U$ is open. Therefore, setting 
\[
	\Phi(\mu^N) := \begin{pmatrix}
		(\partial_t\theta_{\mu^N}, \Delta_x\theta_{\mu^N})_{L^2(\mc X)}\\
		\|\Delta_x\theta_{\mu^N}\|^2_{L^2(\mc X)}
	\end{pmatrix}
\]
and the function $F : \R \times (0,+\infty) \to \R$, $(\alpha,\beta) \mapsto \alpha/\beta$, we get that $Q(\mu^N) = F(\Phi(\mu^N))$. Now, since for $ v \in \msr H$, it holds that $\partial_t\theta_v = \langle \partial_t\sigma,v\rangle$ and $\Delta_x\theta_v = \langle \Delta_x\sigma,v\rangle$, a straightforward calculation combined with Theorem~\ref{thm:functional-delta-method} implies that the limiting fluctuation process of observables is given by
\[
		D\Phi(\bar\mu_s)(\bar\eta_s) = \begin{pmatrix}(\partial_t\theta_{\bar\eta_s},\Delta_x\theta_{\bar\mu_s})_{L^2(\mc X)} + (\partial_t\theta_{\bar\mu_s},\Delta_x\theta_{\bar\eta_s})_{L^2(\mc X)}\\
		2(\Delta_x\theta_{\bar\mu_s},\Delta_x\theta_{\bar\eta_s})_{L^2(\mc X)}
		\end{pmatrix},
\]
where $\bar\mu_t$ and $\bar\eta_t$ are respectively the mean field limit and its associated limiting fluctuation process provided by Theorem~\ref{thm:mean-field-lln} and Theorem~\ref{thm:mean-field-clt} of Section~\ref{sec:preliminary-results}. Now since $\mc U$ is open and $\mu_s^N \to \bar\mu_s$ in probability in $\msr H$, the condition $\bar\mu_s \in \mc U$ implies
\[
    \P(\mu_s^N\in\mc U)\longrightarrow 1.
\]
On the event $\{\mu_s^N \notin \mc U\}$, we may define $Q(\mu_s^N)$ arbitrarily, for instance by setting $Q(\mu_s^N) = 0$. Since the probability of the above complementary event tends to zero, the preceding expansion holds with probability tending to one, namely this arbitrary extension has no effect on the limiting distribution. The functional Delta method provided by Theorem~\ref{thm:functional-delta-method} may therefore be applied locally to $Q$ at $\bar\mu_s$. Thus, whenever $\Delta_x\theta_{\bar\mu_s} \neq 0$ in $L^2(\mc X)$, Theorem~\ref{thm:functional-delta-method} and the chain rule yield
\[
    \sqrt{N}\big(Q(\mu_s^N) - Q(\bar\mu_s)\bigr) \xrightarrow[N \to +\infty]{\mathscr L} DQ(\bar\mu_s)(\bar\eta_s),
\]
where for every $\mu, \nu \in \mc U \times \msr H$ we have 
\[
	\begin{split}
		DQ(\mu)(\nu) = \frac{(\partial_t\theta_{\mu},\Delta_x\theta_{\nu})_{L^2(\mc X)} + (\partial_t\theta_{\nu},\Delta_x\theta_{\mu})_{L^2(\mc X)}}{\|\Delta_x\theta_{\mu}\|^2_{L^2(\mc X)}}\\
		 - 2\frac{(\partial_t\theta_{\mu},\Delta_x\theta_{\mu})_{L^2(\mc X)}(\Delta_x\theta_{\mu},\Delta_x\theta_\nu)_{L^2(\mc X)}}{\|\Delta_x\theta_{\mu}\|^4_{L^2(\mc X)}}.
	\end{split}
\]
The limiting variable is Gaussian, and its variance is given by Corollary~\ref{coro:covariance-formula} whenever $DQ(\bar\mu_s)$ admits the required representation. Moreover, since $Q = F\circ\Phi$ on $\mc U$, the effective diffusivity is locally identifiable from the two components of $\Phi$, even though the full parameter distribution need not be identifiable.
\end{ex}

The remainder of the paper is organized as follows. In Section~\ref{sec:mean-field-dynamics}, we introduce the mean field formulation of shallow neural network training and describe the stochastic gradient dynamics considered throughout the paper. Section~\ref{sec:preliminary-results} recalls the functional framework and the main results on the mean field limit and its asymptotic fluctuations that provide the starting point of our analysis. In Section~\ref{sec:proofs}, we prove the propagation of the fluctuation theory to nonlinear observation functionals, derive the corresponding covariance representation, and establish the factorization criterion for local identifiability. Finally, Appendix~\ref{app:auxiliary-results} provides the functional analytic arguments underlying the Hilbert-Schmidt embeddings between weighted Sobolev spaces used throughout the analysis.

\section{Mean-field dynamics of shallow neural networks training}\label{sec:mean-field-dynamics}

In recent years, considerable effort has been devoted to the mathematical analysis of neural network training dynamics in the large-width regime. This has led to a remarkably rich literature at the intersection of stochastic analysis, interacting particle systems, optimal transport, and mean-field theory, we refer e.g. to the works by F.Bach \& L.Chizat~\cite{bach-chizat-2023}, by  A.Descours, A.Guillin, M.Michel \& B.Nectoux~\cite{descours-guillin-michel-nectoux-2024}, by A.Descours, A.Guillin, G.Lacour, M.Michel, B.Nectoux \& P.Stos~\cite{descours-guillin-lacour-michel-nectoux-stos-2026}, by B.Geshkovski, C.Letrouit, Y.Polyanski \& P.Rigollet~\cite{geshkovski-letrouit-polyanski-rigollet-2025}, by A.Guillin, B.Nectoux \& P.Stos~\cite{guillin-nectoux-stos-2026}, by S.Mei, A.Montanari \& P.-M.Nguyen~\cite{mei-montanari-nguyen-2018} and by J.Sirignano \& K.Spiliopoulos~\cite{sirignano-spiliopoulos-2020}, as to the books by F.Bach~\cite{bach-2024} or by J.Sirignano, R.Sowers \& K.Spiliopoulos~\cite{sirignano-sowers-spiliopoulos-2025} and references therein for an overview of those approaches. One of the main mathematical difficulties lies in the non-convex nature of the optimization landscape associated with neural network training. 
In high dimension, the stochastic gradient descent dynamics governing the evolution of the parameters may therefore appear analytically intractable at first sight. 

A fruitful approach consists in representing the collection of network parameters through its associated empirical measure and studying the corresponding evolution in the space of probability measures. In this formulation, the training dynamics can be interpreted as a mean-field interacting particle system, whose large-width limit is described by a deterministic evolution equation on measures. Such a perspective provides a mathematically tractable framework for the rigorous analysis of neural network training in the infinite-width regime. Beyond the deterministic mean-field limit itself, an important question concerns the description of finite-width effects and the asymptotic fluctuations of the empirical measure around its limiting dynamics.

More precisely, we consider the following single-hidden-layer neural network
\[
	g_W^N(x):=\frac{1}{N}\sum_{i=1}^N \sigma^*(W^i,x),
\]
where $x \in \R^m$ denotes the input data, $y \in \R$ the associated label, and $\sigma^*: \R^d \times \R^m \to \R$ is the activation function. The parameter vector $W = (W^1,\dots,W^N) \in (\R^d)^N$ collects the weights of the network, while $N$ denotes the number of neurons. We assume that the data are sampled from an unknown probability distribution $(x,y) \sim \pi$, where $\pi$ is not explicitly known. The objective of the training procedure is to determine weights $W$ such that the network prediction $g_W^N(x)$ approximates the target value $y$ in a suitable sense, which is given through an optimization problem. In the present work, we consider the quadratic loss
\begin{equation}\label{eq:loss-functional}
	\mc R(W) = \frac{1}{2} \mbb E_{(x,y)\sim\pi}\left[|g_W^N(x)-y|^2\right],
\end{equation}
so that the training of the weights therefore amounts to solve
\begin{equation}\label{eq:nn-training}
	\min_{W\in(\R^d)^N}\mc R(W).
\end{equation}

In practice, the network is trained through the stochastic gradient descent algorithm (see e.g. the book by F.Bach~\cite[Chapter 5, Section 4]{bach-2024}, the book by J.Sirignano, R.Sowers \& K.Spiliopoulos~\cite[Chapters 7 and 18]{sirignano-sowers-spiliopoulos-2025}). Denoting by $\alpha > 0$ the learning rate, the dynamics of the weights is given, for $k \geq 0$ and $i\in\{1,\dots,N\}$, by
\begin{equation}\label{eq:sgd}\tag{SGD}
	\begin{cases}
		W_{k+1}^i = W_k^i + \displaystyle\frac{\alpha}{N|B_k|}\sum_{(x,y)\in B_k}(y-g_{W_k}^N(x))\nabla_W\sigma^*(W_k^i,x)+\frac{\ve_k^i}{N^\beta},\\
		W_0^i\sim \mu_0,
\end{cases}
\end{equation}
where $\ve_k^i\sim \mc{N}(0,I_d)$ and $\mu_0 \in \mc P(\R^d)$.

Here, $B_k$ denotes the mini-batch used at iteration $k$, namely a finite random subset of the dataset sampled according to the underlying distribution $\pi$. 
The use of mini-batches is standard in modern machine learning, as it provides computationally efficient approximations of the full gradient while preserving the stochastic nature of the training dynamics (we refer the interested reader to the book by J.Sirignano, R.Sowers \& K.Spiliopoulos~\cite[Chapter~10]{sirignano-sowers-spiliopoulos-2025}).  The additional Gaussian perturbation models an external source of noise in the optimization procedure.  Such noisy stochastic gradient dynamics naturally arise in several contexts, including regularization mechanisms, stochastic approximations of optimization procedures, or effective descriptions of finite-width training effects. The parameter $\beta>0$ governs the asymptotic strength of the noise in the large-width regime.

Observe that the stochastic gradient descent dynamics \eqref{eq:sgd} does not directly involve the full gradient of the risk functional \eqref{eq:loss-functional}. Since the underlying distribution $\pi$ is unknown in practice, the gradient of the risk cannot be computed explicitly. Instead, the algorithm relies on a random empirical approximation, namely the $i$-th component of this gradient is approximated using the mini-batch estimator
\[
    \widehat{\nabla_{W^i}\mc R}(W_k;B_k) := \frac{1}{N|B_k|}\sum_{(x,y)\in B_k} \big(g_{W_k}^N(x) - y\big)\nabla_w\sigma^*(W_k^i,x).
\]
which acts as a stochastic estimator of the expected gradient associated with the loss functional. Thus, the update rule can be written as
\[
    W_{k+1}^i = W_k^i - \alpha\,\widehat{\nabla_{W^i}\mc R}(W_k;B_k) + \frac{\ve_k^i}{N^\beta}.
\]

While this stochastic approximation circumvents the explicit dependence on the unknown distribution $\pi$, the resulting optimization dynamics nevertheless remains highly nonlinear with respect to the weights $W$. This nonlinearity reflects the non-convex nature of the original optimization problem~\eqref{eq:nn-training}, and makes the direct analysis of the large-width regime particularly delicate. A key observation is that the network output can naturally be rewritten in terms of the empirical distribution of the parameters. More precisely, introducing the empirical measure
\[
	\nu^N = \frac{1}{N}\sum_{i=1}^N\delta_{W^i},
\]
we may rewrite the network prediction as
\[
	g_W^N(x) = \langle \sigma^*(\cdot,x),\nu^N\rangle,
\]
where $\langle \cdot, \cdot \rangle$ is a duality bracket on a suitable measure space. This reformulation allows the training dynamics to be interpreted in terms of the empirical distribution of the parameters in $\mathbb R^d$. Therefore, the original optimization problem on the weights given by \ref{eq:nn-training} induces a corresponding optimization problem on probability measures. Formally, introducing
\[
	g_\nu(x) = \int_{\R^d}\sigma^*(w,x)\,d\nu(w),
\]
we may rewrite the loss functional~\eqref{eq:loss-functional} as
\[
	\mathcal R(\nu) = \frac{1}{2}\mbb E_{(x,y)\sim\pi}\left[ |g_\nu(x)-y|^2\right],
\]
letting us consider the relaxed problem
\begin{equation}\label{eq:nn-training-relaxed}
	\min_{\nu \in \mc P_\gamma(\R^d)} \mc R(\nu).
\end{equation}
In the above formulation, the optimization problem is no longer posed on the finite-dimensional parameter vector space of the weights, but on the space of probability measures $\mc P_\gamma(\R^d)$, which is the space of probability measures with bounded moment of order $\gamma > 0$. 

Moreover, due to the linearity of $g_\nu$ with respect to $\nu$, the functional $\mc R(\nu)$ becomes convex in the measure variable. This relaxed variational perspective has played an important role in the mathematical analysis of wide neural networks, we refer for instance to the book of F.Bach~\cite[Chapter 4]{bach-2024} and references therein.

Beyond its variational interpretation, this measure-valued formulation naturally leads to a mean-field description of the stochastic gradient descent dynamics. In the large-width regime, the empirical measure evolves as a system of interacting particles whose asymptotic behavior can be analyzed through mean-field techniques, in close analogy with classical models from statistical physics.  Such an approach provides a rigorous framework for the study of infinite-width neural network training and its associated fluctuation theory. We refer for instance to the monograph of H.Spohn~\cite{spohn-1991} for background on mean-field interacting particle systems.

To describe such dynamics, we introduce the rescaled process
\[
	\mu_t^N := \nu_{\lfloor Nt\rfloor}^N = \frac{1}{N}\sum_{i=1}^N\delta_{W_{\lfloor Nt\rfloor}^i},
\]
and rewrite the stochastic gradient descent algorithm in terms of $\mu_t^N$ instead of $\nu_k^N$, without any specific change. Under some reasonable assumptions, one can expect to obtain a deterministic limit for the infinite-width regime, which is the object of Theorem~\ref{thm:mean-field-lln} recalled in Section~\ref{sec:preliminary-results} below.

Nevertheless, it should be noted that, from a computational standpoint, training a neural network never reaches such a limit but is restricted to a hidden layer of very large width $N$. It is therefore natural to ask to what extent the process $\mu_t^N$ fluctuates around its mean-field limit $\bar\mu_t$. Defining the fluctuation process $\eta_t^N = \sqrt{N}(\mu_t^N - \bar\mu_t)$, we can establish the existence of a limiting fluctuation process $\bar\eta_t$ that satisfies a stochastic partial differential equation. This result is given by Theorem~\ref{thm:mean-field-clt} below.

Interestingly, the stochastic evolution equation \eqref{eq:weak-fluctuation-edps} is linear in $\bar\eta$ and is driven by the Gaussian process $\msr G$. It is therefore natural to expect the limiting fluctuation process $\bar\eta$ to remain Gaussian, as described in the book by J.Sirignano, R.Sowers \& K.Spiliopoulos~\cite[Section 20.4]{sirignano-sowers-spiliopoulos-2025}. In order to obtain a complete uncertainty quantification result, one must moreover characterize its covariance explicitly. This is the purpose of Theorem~\ref{thm:mean-field-uq}, recalled in the next section.

\section{Preliminary results}\label{sec:preliminary-results}

In this section, we recall the mathematical framework underlying the mean-field analysis of shallow neural networks. We first introduce the functional spaces in which the empirical measures and their fluctuations will be studied in subsection~\ref{ssec:functional-framework}. In subsection~\ref{ssec:assumptions} we then summarize the assumptions required for the existing law of large numbers and functional central limit theorem, and recall these results in a form adapted to the present work. Finally, we state the explicit characterization of the limiting fluctuation process established in the work by A.Descours, A.Guillin, G.Lacour, M.Michel, B.Nectoux \& P.Stos~\cite[Theorem~3]{descours-guillin-lacour-michel-nectoux-stos-2026} in Theorem~\ref{thm:mean-field-uq}, which constitutes the starting point of our analysis.

\subsection{Functional framework}\label{ssec:functional-framework}

We first recall the functional framework inherent to the results. The state of the system is described by probability measures on $\R^d$, while the fluctuation processes will be viewed as distribution-valued paths. We therefore introduce, in a systematic way, the spaces of probability measures, the  path spaces, and the weighted Sobolev spaces used below.

\begin{defi}
	Let $k\geq 1$. We denote by $\mc P_k(\R^d)$ the set of
	Borel probability measures on $\R^d$ with finite $k$-th moment,
	namely
	\[
		\mc P_k(\R^d) := \left\{\mu\in \mc P(\R^d)\,\mid\,\int_{\R^d} |x|^k\,\di\mu(x) < +\infty\right\}.
	\]
	The space $\mc P_k(\R^d)$ is endowed with the Wasserstein distance
	\[
		W_k(\mu,\nu) = \left(\inf_{\pi\in\Pi(\mu,\nu)}\int_{\R^d \times \R^d}|x-y|^k\,\di\pi(x,y)\right)^{1/k},
	\]
	where $\Pi(\mu,\nu)$ denotes the set of probability measures on $\R^d \times \R^d$ with marginals respectively denoted $\mu$ and $\nu$.
\end{defi}

We recall that $W_1(\mu,\nu)\leq W_k(\mu,\nu)$ for all $\mu,\nu \in \mc P_k(\R^d)$ and all $k \geq 1$. Moreover, the Kantorovich-Rubinstein formula gives
\[
	 W_1(\mu,\nu) = \sup_{[f]_{\rm{Lip}}\leq 1}\left|\int_{\R^d} f\,\di\mu - \int_{\R^d} f\,\di\nu\right|,
\]
where the supremum is taken over all real-valued Lipschitz functions $f$ on $\R^d$ with Lipschitz seminorm bounded by one. We refer for instance to the books L.Ambrosio, N.Gigli \& G.Savaré~\cite[Chapter 7]{ambrosio-gigli-savare-2008}, by F.Santambrogio~\cite[Chapter 5]{santambrogio-2015} or by C.Villani~\cite[Chapter 6]{villani-2009} for further background on Wasserstein metric and spaces. 

Throughout this section, the empirical measure process is denoted by $\mu^N := (\mu_t^N)_{t\geq0}$. Under the moment assumptions imposed below, $\mu^N$ is naturally viewed as a random variable with values in the Skorokhod space $\mc D(\R_+,\mc P_q(\R^d))$, for suitable $q \geq 1$.

\begin{defi}
Let $(E,d_E)$ be a metric space. We denote by $\mc D(\R_+,E)$ the space of càdlàg functions from $\R_+$ into $E$, that is, the space of functions $u : \R_+ \to E$ which are right-continuous and admit left limits at every positive time. When $E$ is a Polish space, $\mc D(\R_+,E)$ is endowed with the usual Skorokhod $J_1$ topology.
\end{defi}

Roughly speaking, convergence in $J_1$  topology allows small perturbations of the time variable in order to align jump times, while preserving the chronological order of the jumps. More precisely, a sequence $(x_n)$ converges to $x$ for the $J_1$ topology if there exist increasing homeomorphisms of the time interval converging uniformly to the identity such that the time-changed paths converge uniformly. Thus, the Skorokhod space is the natural state space for empirical measure processes with jumps. We shall use it both with $E = \mc P_q(\R^d)$ and with $E$ a dual weighted Sobolev space.

We mainly refer to the works by V.Bogachev~\cite{bogachev-1998}, by A.Jakubowski~\cite{jakubowski-1986} or the books by P.Billingsley~\cite[Chapter 3]{billingsley-1999}, by S.Ethier \& T.Kurtz~\cite[Chapter 3]{ethier-kurtz-2005}, by J.Jacod \& A.Shiryaev~\cite[Chapter VI]{jacod-shiryaev-2003} and by D.Pollard~\cite[Chapters V and VI]{pollard-1984} for a detailed construction as well as an overview of the topological poperties for such spaces. 

The fluctuation processes are distribution-valued. To make this precise, we use weighted Sobolev spaces. The weights allow test functions and their derivatives to grow at most polynomially at infinity, while preserving a Hilbert space structure. 

\begin{defi}\label{def:sobolev-weighted}
	Let $J\in \N$ and $a \geq 0$. For $g \in C_c^\infty(\R^d)$, set
\[
	\|g\|_{\mc H^{J,a}}^2 := \sum_{|\alpha|\leq J}\int_{\R^d}\frac{|D^\alpha g(x)|^2}{1+|x|^{2a}}\dx.
\]
	We define $\mc H^{J,a}(\R^d)$ as the completion of $C_c^\infty(\R^d)$ for this norm.
\end{defi}

The space $\mc H^{J,a}(\R^d)$ as defined above is a Hilbert space. The associated inner product is denoted by $( \cdot,\cdot)_{J,a}$. We denote by $\mc H^{-J,a}(\R^d)$ its topological dual. If $\Phi \in \mc H^{-J,a}(\R^d)$ and $f \in \mc H^{J,a}(\R^d)$, we write
\[
	\langle f,\Phi\rangle_{J,a} := \Phi[f].
\]
Of course, one can apply Riesz representation theorem to identify $\mc H^{J,a}$ with its topological dual. Under this identification, the duality pairing $\langle \cdot, \cdot \rangle_{J,a}$ coincides with the Hilbert inner product $( \cdot, \cdot)_{J,a}$. Whenever no confusion may arise, we omit both this identification and the indices from the notation.

\begin{defi}\label{def:holder-spaces}
	Let $J \in \N$ and $a \geq 0$. We define $\mc C^{J,a}(\R^d)$ as the space of functions $f : \R^d \to \R$ with continuous partial derivatives up to order $J$ such that, for every multi-index $\alpha$ with $|\alpha|\leq J$
	\[
		\lim_{|x|\to\infty}\frac{|D^\alpha f(x)|}{1+|x|^a} = 0.
	\]
	It is endowed with the norm
	\[
		\|f\|_{\mc C^{J,a}} := \sum_{|\alpha|\leq J}\sup_{x \in \R^d}\,\frac{|D^\alpha f(x)|}{1+|x|^a}.
	\]
\end{defi}

We also denote by $\mc C_b(\R^d)$ the space of bounded continuous functions on $\R^d$, endowed with the supremum norm, and by $\mc C_b^\infty(\R^d)$ the space of smooth functions whose derivatives of all orders are bounded. Observe that if $b > d/2$, then Hölder inequality leads to
\[
	\mc C_b^\infty(\R^d) \subset	\mc H^{J,b}(\R^d)
\]
for every $J \in \mbb N$. 

The weighted Sobolev spaces introduced above retain many of the fundamental properties of the classical Sobolev spaces. In particular, they satisfy the embedding results recalled below, which play a central role in the functional analysis of the fluctuation process. The first one is a weighted Sobolev embedding into spaces of functions with polynomial growth, while the second one is a Hilbert-Schmidt embedding between weighted Sobolev spaces.

\begin{prop}[Weighted Sobolev embeddings]\label{ppt:sobolev-weighted-embeddings}
Let $j \in \N$, $a \geq 0$, $b > d/2$, and $\ell>d/2$. Then the following embeddings hold
\[
	\mc H^{\ell+j,a}(\R^d) \hookrightarrow \mc C^{j,a}(\R^d),
\]
and
\[
	\mc H^{\ell+j,a}(\R^d) \overset{\rm{H.S.}}{\hookrightarrow} \mc H^{j,a+b}(\R^d).
\]
Here $\overset{\rm{H.S.}}{\hookrightarrow}$ means that the canonical embedding is a Hilbert-Schmidt operator.
\end{prop}

The continuous embedding into $\mc C^{j,a}$ follows from the classical Sobolev embedding theorem combined with the polynomial weight. For background on Sobolev embeddings and weighted Sobolev spaces, we refer for instance
to the works by B.Fernandez \& S.M{\'e}l{\'e}ard~\cite{fernandez-meleard-1997} or by M.Izuki, T.Nogayama, T.Noi \& Y.Sawano~\cite{izuki-nogayama-noi-sawano-2023} as well as to the books by R.Adams \& J.Fournier~\cite[Section 4.56]{adams-fournier-2003}, by J.Heinonen, T.Kilpel{\"a}inen \& O.Martio~\cite[Chapter 1]{heinonen-kilpelainen-olli-1993}, by A.Kufner~\cite{kufner-1985} or by H.Triebel~\cite[Section 7.2]{triebel-2010}. The only diffcult point of Proposition~\ref{ppt:sobolev-weighted-embeddings} concerns the Hilbert-Schmidt property, which does not directly follow from a representation of Bessel potentials. Since this point is often used implicitly in the literature, we recall the argument in Appendix~\ref{app:auxiliary-results}, but we refer to the work by T.Krainer \& B.-W.Schulze~\cite[Chapter~4, Theorem~4.1.5-d.]{krainer-schulze-2002} for more general considerations.

We first introduce the indices used in the law of large numbers given by
\[
	L=\left\lceil\frac{d}{2}\right\rceil + 3, \quad \gamma = 4\left\lceil\frac{d}{2}\right\rceil + 5,\quad \text{and}\quad \gamma^* = \gamma + 1.
\]
Since $L-2 > d/2$, Proposition~\ref{ppt:sobolev-weighted-embeddings} yields the
continuous embedding
\[
	\mc H^{L,\gamma}(\R^d) \hookrightarrow \mc C^{2,\gamma}(\R^d).
\]
Moreover, since $\gamma^* > \gamma$, one has
\[
	\mc C^{2,\gamma}(\R^d) \hookrightarrow \mc C^{2,\gamma^*}(\R^d),
\]
and therefore
\[
	\mc H^{L,\gamma}(\R^d) \hookrightarrow \mc C^{2,\gamma^*}(\R^d).
\]

Since every twice continuously differentiable function with polynomial growth of order $\gamma^*$ is, in particular, a continuous function with the same polynomial growth,
\[
	\mc C^{2,\gamma^*}(\R^d) \hookrightarrow \mc C^{0,\gamma^*}(\R^d).
\]
Passing to the dual spaces therefore gives the continuous embeddings
\[
	\mc C^{0,\gamma^*}(\R^d)' \hookrightarrow \mc C^{2,\gamma^*}(\R^d)' \hookrightarrow \mc H^{-L,\gamma}(\R^d).
\]
Consequently, every finite Borel measure with finite $\gamma^*$-th moment naturally defines an element of $\mc H^{-L,\gamma}(\R^d)$. This space is the negative-order Sobolev space naturally associated with the law of large numbers. The fluctuation central limit theorem is formulated in a weaker negative-order Sobolev space, which we now introduce. Following the notation introduced in the work by A.Descours, A.Guillin, M.Michel \& B.Nectoux~\cite{descours-guillin-michel-nectoux-2024}, let
\[
    J_0 \geq 4\bigg\lceil\frac{d}{2}\bigg\rceil + 8, \quad j_0 = \bigg\lceil\frac{d}{2}\bigg\rceil + 2,
\]
and set
\[
    \msr H := \mc H^{-J_0+1,j_0}(\R^d).
\]
The embeddings used in the fluctuation analysis yield $\mc H^{J_0-1,j_0}(\R^d) \hookrightarrow \mc H^{L,\gamma}(\R^d)$, so that by duality,
\begin{equation}\label{eq:embedding-lln-clt-spaces}
    \mc H^{-L,\gamma}(\R^d) \hookrightarrow \msr H.
\end{equation}
Thus, the empirical measures and the mean field limit may be considered as elements of $\msr H$, while the fluctuation central limit theorem is naturally stated in this latter space.

\subsection{Assumptions and known results}\label{ssec:assumptions}

We begin by recalling the assumptions under which the mean-field limit and the corresponding fluctuation theorem have been established. Throughout the paper, the stochastic gradient dynamics generates a natural filtration describing the information available up to each optimization step. For every $N \geq 1$, we define
\begin{equation*}\label{eq.filtration}
	\mc F_0^N = \sigma\bigg(\{W_0^i\}_{i=1}^N\bigg),
\end{equation*}
and, for every $k \geq 1$,
\begin{equation*}
	\mc F_k^N = \sigma\bigg(W_0^i,\{B_j\}_{j=0}^{k-1},\{\varepsilon_j^i\}_{j=0}^{k-1}\bigg),\; i\in\{1,\ldots,N\}.
\end{equation*}
In the above, $\sigma$ denotes the $\sigma$-algebra generated by the corresponding random variables. The assumptions used throughout the paper are summarized below.
\begin{enumerate}[label= {\bf A\arabic*}]
	\item\label{ass:data}(Data sampling). For every $k \geq 0$, the collection $\big((x_k^n,y_k^n)\big)_{n \geq 1}$ consists of independent and identically distributed random variables with common distribution $\pi \in \mc P(\mc X \times \mc Y)$. Moreover,
		\begin{itemize}
			\item the collections $\big((x_k^n,y_k^n)\big)_{n\geq1}$ are
			independent for different values of $k$,
			\item for every $k\geq0$, the pair $\big(|B_k|,\big((x_k^n,y_k^n)\big)_{n\geq1}\big)$ is independent of $\mc F_k^N$,
			\item for every $q,k \geq 0$, $|B_q|$ is independent of $\big((x_k^n,y_k^n)\big)_{n \geq 1}$,
			\item the output variable satisfies $\mbb E\big[|y|^{16\gamma^*}\big]<+\infty$.
		\end{itemize}
	\item\label{ass:sigma}(Regularity of the activation function). The activation function $\sigma^*: \R^d \times \mc X \longrightarrow \R$ belongs to $\mc C_b^\infty(\R^d \times \mc X)$.
	\item\label{ass:init}(Initialization). The initial parameters $\{W_0^i\}_{i=1}^N$ are i.i.d. with common law $\mu_0 \in \mc P(\R^d)$ satisfying $\mbb E\big[|W_0^1|^{8\gamma^*}\big] < +\infty$.
	\item\label{ass:noise}(Exploration noise). The family $\big(\ve^i_k\big)_{1 \leq i \leq N}^{k \geq 0}$ consists of independent Gaussian random vectors with distribution $\mc N(0,I_d)$. Moreover, for every $k \geq 0$ and every $i\in\{1,\ldots,N\}$ the random variable $\ve_k^i$ is independent of $\mc F_k^N$.
\end{enumerate}

The previous assumptions are standard in the mean-field analysis of
stochastic gradient methods. For the reader's convenience, we briefly
comment on their respective roles. Assumption~\ref{ass:data} describes the learning framework. At each optimization step, the algorithm receives a fresh mini-batch, sampled independently of both the past trajectory and the current state of the network. The moment assumption on the output variable is used to derive uniform moment estimates for the system and its mean-field limit. Assumption~\ref{ass:sigma} is a regularity assumption on the activation function. It ensures that the drift of the limiting mean-field equation is sufficiently smooth and that all the derivatives appearing in the fluctuation analysis are well defined. Assumption~\ref{ass:init} states that the neurons are initially exchangeable. The moment condition propagates along the dynamics and is required in several tightness and stability arguments. Finally, Assumption~\ref{ass:noise} models the exploration noise added to the stochastic gradient dynamics. Its independence with respect to the natural filtration expresses that the perturbation introduced at each optimization step is independent of the past history of the algorithm.

We are now in a position to recall the mean-field limit established in the work by A.Descours, A.Guillin, M.Michel \& B.Nectoux~\cite[Theorem 1]{descours-guillin-michel-nectoux-2024}. We highlight that similar results are proved in the works by Z.Chen, G.M.Rotskoff, J.Bruna \& E.Vanden-Eijnden~\cite{chen-rotskoff-bruna-vanden-eijnden-2020}, by S.Mei, A.Montanari \& P.-M.Nguyen~\cite{mei-montanari-nguyen-2018}, or in particular by J.Sirignano \& K.Spiliopoulos~\cite{sirignano-spiliopoulos-2020}. Roughly speaking, as the width of the neural network tends to infinity, the empirical distribution of the parameters converges towards a deterministic measure-valued trajectory governed by a nonlinear evolution equation.

\begin{thm}[Mean-field limit]\label{thm:mean-field-lln}
	Assume that $\beta>\frac{1}{2}$ and that Assumptions~\ref{ass:data}--\ref{ass:noise} hold. Then the empirical measure process $(\mu^N)_{N \geq 1}$ converges in probability, in $\mc D(\R_+,\mc P_\gamma(\R^d))$, towards a deterministic trajectory $\bar\mu \in \mc C(\R_+,\mc P_1(\R^d))$.

Moreover, $\bar\mu$ is the unique solution in $\mc C(\R_+,\mc P_1(\R^d))$ of the following weak formulation of the nonlinear measure-valued equation given for every $f \in \mc C_b^\infty(\R^d)$ and every $t \geq 0$ by
\begin{equation}\label{eq:mean-field-limit}
	\langle f,\bar\mu_t\rangle = \langle f,\mu_0\rangle + \alpha\int_0^t\int_{\mc X \times \mc Y}\Big(y - \langle\sigma^*(\cdot,x),\bar\mu_s\rangle\Big) \langle \nabla f\cdot\nabla\sigma^*(\cdot,x), \bar\mu_s\rangle \, \di\pi(x,y)\, \di s.
\end{equation}
Furthermore, $\bar\mu \in \mc C\big(\R_+,\mc H^{-L,\gamma}(\R^d)\big)$, and equation~\eqref{eq:mean-field-limit} extends to every test function $f$ which belongs to $\mc H^{L,\gamma}(\R^d)$.
\end{thm}

The parameter $\beta$ is the scaling exponent of the exploration noise introduced in the stochastic gradient dynamics \eqref{eq:sgd}. The condition $\beta > \frac{1}{2}$ ensures that this perturbation vanishes at the mean-field scale and therefore does not affect the deterministic limiting dynamics. The limiting equation~\eqref{eq:mean-field-limit} is a nonlinear transport equation on the space of probability measures whose nonlinearity stems from the dependence of the drift on the current distribution $\bar\mu_t$, which is characteristic of McKean-Vlasov type dynamics.

Theorem~\ref{thm:mean-field-lln} states that the limiting trajectory is continuous for the $W_1$ topology. Moreover, as stated in the work by A.Descours, A.Guillin, M.Michel \& B.Nectoux~\cite[Corollary~1]{descours-guillin-michel-nectoux-2024} it is possible to show that $\bar\mu \in C\big(\R_+,\mc H^{-L,\gamma}(\R^d)\big)$, and that the limiting equation extends to test functions in $\mc H^{L,\gamma}(\R^d)$. The higher-order moment estimates obtained in the same work allow one to strengthen the Wasserstein regularity of the limiting trajectory. More precisely it is in fact continuous for the $W_\gamma$ topology in which the functional law of large numbers is formulated.

\begin{coro}[Continuity of the mean field limit in $W_\gamma$ topology]\label{coro:Wgamma-continuity} Under the assumptions of Theorem~\ref{thm:mean-field-lln}, the deterministic limiting trajectory satisfies $\bar\mu \in C\big(\R_+,\mc P_\gamma(\R^d)\big)$. Therefore, for every $t \geq 0$, the following convergence holds in $\mc P_\gamma(\R^d)$:
\[
    \mu_t^N \xrightarrow[N \to +\infty]{\mbb P} \bar\mu_t.
\]
\end{coro}

\begin{proof}
Let us start by proving the following Wasserstein interpolation estimate:  let $1 \leq r < p < q < +\infty$ and let $\theta \in (0,1)$ be defined by
\[
    \frac{\theta}{r}+\frac{1-\theta}{q} = \frac{1}{p} .
\]
Then, for every $\mu,\nu \in \mc P_q(\R^d)$,
\begin{equation}\label{interpolation_Wasserstein}
    W_p(\mu,\nu) \leq  2^{\frac{(q-1)(1-\theta)}{q}}\bigg(\int_{\R^d}|x|^q\,\di\mu(x) + \int_{\R^d}|x|^q\,\di\nu(x)\bigg)^{\frac{1-\theta}{q}} W_r(\mu,\nu)^\theta.
\end{equation}
To prove this, let $\mu,\nu \in \mc P_q(\R^d)$ and consider their optimal coupling $\pi_r$ for $W_r$. Denoting by $f$ the function $(x,y)\mapsto |x-y|$, we have, thanks to the log-convexity of Lebesgue norms applied to $L^r(\pi_r)$ and $L^q(\pi_r)$ norms, $\|f\|_{L^p} \le \|f\|_{L^r}^\theta\|f\|_{L^q}^{1-\theta}$. Inequality \eqref{interpolation_Wasserstein} follows from $ W_p(\mu,\nu)\le \|f\|_{L^p}$, $ W_r(\mu,\nu)^\theta = \|f\|_{L^r}^\theta$ and $\|f\|_{L^q}^q\le 2^{q-1}\bigg(\dst\int_{\R^d}|x|^q\,\di\mu(x) + \dst\int_{\R^d}|x|^q\,\di\nu(x)\bigg)$.

Then, let $T>0$ and $q\in(\gamma,\gamma+1)$. By \cite[Corollary~13]{descours-guillin-michel-nectoux-2024}, it holds 
\begin{equation}\label{eq_Cor13_JMLR}
    \underset{N \geq 1}{\mathrm{sup}}\, \mbb E\bigg[\underset{t \in [0,T+1]}{\mathrm{sup}}\,\int_{\R^d}|w|^q\,\di\mu_t^N(w)\bigg] < +\infty.
\end{equation}
Applying Skorokhod representation theorem, see e.g. the books by P.Billingsley~\cite[Theorem 6.7]{billingsley-1999} or by V.Bogachev~\cite[Theorem 8.5.4]{bogachev-2007}, there exist a probability space and some $\mathcal D(\R_+,\mathcal P_\gamma(\R^d))$-valued random elements $(M^N)_{N \ge 1}$ defined on that space such that $M^N\overset{\mathscr L}{=}\mu^N$ for all $N\ge1$ and  $M^N \underset{\to}{N \to +\infty} \bar\mu$ almost surely in $\mathcal D(\R_+,\mathcal P_\gamma(\R^d))$. Denote $\mathtt D_{T+1}\subset [0,T+1]$ the set of discontinuity points of $\bar\mu$, restricted to $[0,T+1]$. By a result from the book by P.Billingsley~\cite[Chapter~3, Lemma~1]{billingsley-1999}, this set is at most countable.  Since convergence in Skorokhod topology implies convergence at any continuity point of the limit (see e.g. the book by P.Billingsley~\cite[Equation (12.14)]{billingsley-1999}), it holds almost surely for all $t \in [0,T+1]\backslash \mathtt D_{T+1}$, $W_\gamma(M_t^N,\bar\mu_t)\to 0$.       
Now, fix $t\in [0,T+1] \backslash \mathtt D_{T+1}$. Then, see e.g. the book by L.Ambrosio, N.Gigli \& G.Savaré~\cite[Lemma 5.1.7]{ambrosio-gigli-savare-2008}, it holds almost surely
\begin{equation*}
\int_{\R^d}|w|^q\,\di \bar\mu_t(w)\le\liminf_N\int_{\R^d}|w|^q\,\di M_t^N(w).
\end{equation*} 
Taking the expectation and applying Fatou's lemma, we obtain
\begin{align}\label{eq:Fatou+Ambrosio}
\int_{\R^d}|w|^q\,\di \bar\mu_t(w)&\le\liminf_N\mathbb E\Big[\int_{\R^d}|w|^q\,\di \mu_t^N(w)\Big]\nonumber\\
&\le \liminf_N\mathbb E\Big[\sup_{s\in[0,T+1]}\int_{\R^d}|w|^q\,\di \mu_s^N(w)\Big]\nonumber\\
&\le \underset{N \geq 1}{\mathrm{sup}}\mathbb E\Big[\sup_{s\in[0,T+1]}\int_{\R^d}|w|^q\,\di \mu_s^N(w)\Big].
\end{align}
Hence, recalling \eqref{eq_Cor13_JMLR}, we just proved that 
\begin{equation*}
\sup_{t\in[0,T+1]\backslash \mathtt D_{T+1} }\int_{\R^d}|w|^q\,\di \bar\mu_t(w) < +\infty.
\end{equation*}
Finally, if $t\in\mathtt D_{T+1} \cap [0,T]$, consider a sequence $(t_n) \subset [0,T+1]\backslash \mathtt D_{T+1}$ with $t_n \to t$, $t_n > t$. By right continuity, $ W_\gamma(\bar\mu_t,\bar\mu_{t_n})\to 0$, and hence, see e.g. in the book by L.Ambrosio, N.Gigli \& G.Savaré~\cite[Lemma 5.1.7]{ambrosio-gigli-savare-2008}, it follows 
\[
	\dst\int_{\R^d}|w|^q\,\di \bar\mu_t(w)\le \liminf_n \int_{\R^d}|w|^q\,\di \bar\mu_{t_n}(w)\le \sup_{s\in[0,T+1]\backslash \mathtt D_{T+1} }\int_{\R^d}|w|^q\,\di \bar\mu_s(w).
\]
This proves that 
\begin{equation}\label{eq_unif_moment_order_q_barmu}
\sup_{t\in[0,T]}\int_{\R^d}|w|^q\,\di \bar\mu_t(w) < +\infty.
\end{equation}

Then, by \eqref{eq_unif_moment_order_q_barmu}, $\bar\mu_t\in\mathcal P_q(\R^d)$ for any $t\in[0,T]$. Applying \eqref{interpolation_Wasserstein} with $r = 1$, $p = \gamma$, and the above choice $q \in (\gamma,\gamma+1)$, we obtain, for every $s,t\in[0,T]$,
\[
    W_\gamma(\bar\mu_s,\bar\mu_t) \leq C_T W_1(\bar\mu_s,\bar\mu_t)^\theta,
\]
where the constant $C_T$ is provided by \eqref{eq_unif_moment_order_q_barmu}.
The $W_1$-continuity of $\bar\mu$ therefore yields $\bar\mu \in C(\R_+,\mc P_\gamma(\R^d))$, since $T$ is arbitrary. Finally, since evaluation at time $t$ is continuous at every trajectory that is continuous at $t$ for the Skorokhod $J_1$ topology, the following convergence in $\mc D\big(\R_+,\mc P_\gamma(\R^d)\big)$
\[
    \mu^N \xrightarrow[N \to +\infty]{\mbb P} \bar\mu
\]
implies, by the continuous mapping theorem, that for every $t \geq 0$ 
\[
    \mu_t^N \xrightarrow[N \to +\infty]{\mbb P} \bar\mu_t, \text{ in } \mc P_\gamma(\R^d).
\] 
\end{proof}

This improvement concerns the mean-field law of large numbers. The fluctuation central limit theorem recalled below remains formulated in the weaker space
$\msr H = \mc H^{-J_0+1,j_0}(\R^d)$.

The previous theorem establishes the deterministic mean-field limit of the empirical measure process. We now turn to its fluctuations around this limit. Their asymptotic behavior is described by a functional central limit theorem, which requires two additional assumptions.

The following assumptions are only required for the fluctuation analysis. The first one is mainly technical and simplifies several uniform estimates appearing in the proof of the functional central limit theorem. The second one reflects the asymptotic stabilization of the mini-batch size and enters explicitly into the covariance structure of the limiting Gaussian process.

\begin{enumerate}[label= {\bf A\arabic*},resume]
	\item \label{ass:mu0-compact}(Compactly supported initialization). The initial distribution $\mu_0 \in \mc P(\R^d)$ is compactly supported.				\item \label{ass:batch-limit} (Asymptotically constant batch size). The batch size satisfies
	\[
		|B_k| \longrightarrow |B_\infty| \qquad\text{almost surely as }\;  k\to\infty.
	\]
\end{enumerate}

Assumption~\ref{ass:mu0-compact} states that the parameters of the neurons are initially supported in a common compact subset of the parameter space. Although mainly technical, this assumption considerably simplifies the tightness and regularity estimates required for the fluctuation analysis. Assumption~\ref{ass:batch-limit} states that the mini-batch size stabilizes along training. While this assumption has no influence on the deterministic mean-field limit, its limiting value determines the covariance of the Gaussian noise driving the fluctuation process.

The law of large numbers identifies the deterministic trajectory
$\bar\mu$ around which the empirical measure concentrates. To investigate the random deviations from this limit, we introduce the rescaled fluctuation process
\begin{equation}\label{eq:fluctuation-process}
	\eta_t^N = \sqrt N\,(\mu_t^N-\bar\mu_t), \quad t \geq 0.
\end{equation}
The normalization by $\sqrt N$ is the natural one predicted by the central limit theorem. Formally, the fluctuation process is expected to converge towards the solution of a linear stochastic evolution equation obtained by linearizing the mean-field dynamics around the deterministic trajectory $\bar\mu$. Its weak formulation reads as follows. Almost surely for every $f \in \mc H^{J_0,j_0}(\R^d)$, and every $t \in \R_+$,

\begin{equation}\label{eq:weak-fluctuation-edps}
	\begin{split}
		\langle f,\eta_t\rangle - \langle f,\eta_0 \rangle = \alpha\int_0^t\int_{\mc X \times \mc Y} (y-\langle\sigma^*(\cdot,x),\bar \mu_s\rangle)\langle\nabla f\cdot\nabla\sigma^*(\cdot,x),\eta_s\rangle \di\pi(x,y)\di s\\
 - \alpha\int_0^t\int_{\mc X \times \mc Y} \langle\sigma^*(\cdot,x),\eta_s\rangle \langle \nabla f\cdot\nabla\sigma^*(\cdot,x),\bar \mu _s\rangle \di\pi(x,y) \di s + \langle f, \mathscr G_t\rangle.
	\end{split}
\end{equation}

The stochastic forcing term $\mathscr G$ appearing in~\eqref{eq:weak-fluctuation-edps} plays the role of the driving Gaussian noise. Following the terminology introduced by B.Fernandez \& S.Méléard~\cite{fernandez-meleard-1997}, we now define the Gaussian process driving the limiting fluctuation equation. Although referred to as a \emph{G-process}, it is simply an infinite-dimensional centered Gaussian process whose covariance is entirely determined by the deterministic mean-field trajectory and by the asymptotic mini-batch size.

\begin{defi}[G-process]\label{defi:g-process}
	Let $J_0,j_0 \geq 0$. A process $\mathscr G \in \mc C(\R_+,\mc H^{-J_0,j_0}(\R^d))$ is called a \emph{G-process} if the following properties hold.
	\begin{enumerate}
		\item For every integer $k \geq 1$ and every collection of test functions $f_1,\ldots,f_k \in \mc H^{J_0,j_0}(\R^d)$, the projected process
		\[
			\Big(\langle f_1,\mathscr G_t\rangle,\ldots,\langle f_k,\mathscr G_t\rangle\Big)_{t\geq 0}
		\]
		is an $\R^k$-valued centered Gaussian process with continuous
		trajectories and independent increments.
		\item The covariance structure is prescribed by, for every $0\leq s\leq t$, by
		\begin{equation}\label{eq:covariance-g-process}
		\Cov\big ( \langle f_i, \mathscr G_t \rangle,\langle f_j, \mathscr G_s \rangle\big )=\alpha^2 \mbb{E}\left[\frac{1}{|B_{\infty}|}\right]\int_0^s\Cov_\pi(\mathrm Q_v[f_i](x,y),\mathrm Q_v[f_j ](x,y))\, \di v,
		\end{equation}
		where $\mathrm Q_v[f](x,y):=(y-\langle \sigma^*(\cdot, x),\bar \mu_v\rangle)\langle\nabla f\cdot\nabla\sigma^*(\cdot,x),\bar \mu_v\rangle$ for $f \in \mathcal H^{J_0,j_0}(\R^d)$ and $\bar \mu$ is given by Theorem~\ref{thm:mean-field-lln}. 
\end{enumerate}
\end{defi}

Unlike a standard cylindrical Wiener process, the covariance of the G-process in \eqref{eq:weak-fluctuation-edps} depends both on the deterministic mean-field trajectory through the operator $Q_t$ and on the limiting mini-batch size through the factor $\mbb E[|B_\infty|^{-1}]$. Also, note that the stochastic evolution equation~\eqref{eq:weak-fluctuation-edps} is understood in the weak sense. We therefore recall the corresponding notion of weak solution, which provides the appropriate framework for the functional central limit theorem.

\begin{defi}[Weak solution]\label{def:weak-solution-edps} Let $\nu$ be an $\msr H$-valued random variable. A process $\eta \in \mc C\big(\R_+,\msr H\big)$, defined on some probability space, is called a \emph{weak solution} of \eqref{eq:weak-fluctuation-edps} with initial distribution $\nu$ if there exists a G-process $\mathscr G \in \mc C\big(\R_+, \mc H^{-J_0,j_0}(\R^d)\big)$ defined on the same probability space such that \eqref{eq:weak-fluctuation-edps} holds almost surely and $\eta_0$ has the same distribution as $\nu$.

Weak uniqueness is said to hold if any two weak solutions of \eqref{eq:weak-fluctuation-edps}, possibly defined on different probability spaces and having the same initial distribution, have the same law.
\end{defi}

We are now in position to state the functional central limit theorem. It shows that the fluctuations of the empirical measure around the deterministic mean-field limit converge towards the unique solution of a linear stochastic evolution equation driven by the G-process introduced above.

\begin{thm}[Functional central limit theorem]\label{thm:mean-field-clt}
	Assume that $\beta > \frac{3}{4}$,  Assumptions~\ref{ass:data}--\ref{ass:noise} and \ref{ass:mu0-compact}--\ref{ass:batch-limit} hold. Then the fluctuation process $(\eta^N)_{N \geq 1}$, defined in~\eqref{eq:fluctuation-process}, converges in distribution in $\mc D\big(\R_+,\msr H\big)$ towards a continuous process $\bar\eta$ which belongs to $\mc C\big(\R_+, \msr H\big)$.
	
	Furthermore, the limiting process $\bar\eta$ has the same distribution as the unique weak solution $\eta^\star$ of~\eqref{eq:weak-fluctuation-edps} with initial distribution $\nu_0$ in the sense of Definition~\ref{def:weak-solution-edps}. More precisely, $\nu_0$ is the unique $\msr H$-valued Gaussian random variable such that, for every $k \geq 1$ and every $f_1,\ldots,f_k \in \mc H^{J_0 -1, j_0}(\R^d)$,
	\[
		\big(\langle f_1,\bar\eta_0\rangle,\ldots,\langle f_k,\bar\eta_0\rangle \big)^T \sim \mc N \big(0, \Gamma(f_1,\ldots,f_k)\big),
	\]
	where $\Gamma(f_1, \ldots, f_k)$ denotes the covariance matrix of $\big(f_1(W_0^1), \ldots, f_k(W_0^1)\big)^T$.
\end{thm}

The functional central limit theorem shows that, after the $\sqrt N$ rescaling, the random fluctuations of the empirical measure around the deterministic mean-field trajectory become asymptotically Gaussian. In contrast with the law of large numbers given by Theorem~\ref{thm:mean-field-lln}, the limiting dynamics is now stochastic, since it is governed by a linear evolution equation driven by the G-process introduced above. As for Theorem~\ref{thm:mean-field-lln}, the stronger condition $\beta > \frac{3}{4}$ ensures that the exploration noise appearing in the stochastic gradient algorithm remains negligible at the fluctuation scale.

The functional central limit theorem identifies the limiting fluctuation process as the unique weak solution of the stochastic evolution equation~\eqref{eq:weak-fluctuation-edps}. The proof relies on a compactness argument together with Prokhorov theorem, which establishes convergence in distribution without providing an explicit description of the law of the limiting process. Nevertheless, as discussed for instance by
J.Sirignano, R.Sowers \& K.Spiliopoulos~\cite[Chapter~4, Section~20]{sirignano-sowers-spiliopoulos-2025}, the limiting dynamics is expected to remain Gaussian since it solves a linear stochastic equation driven by a Gaussian process. The remaining question was therefore to characterize this Gaussian law explicitly. The approach used in the work by A.Descours, A.Guillin, G.Lacour, M.Michel, B.Nectoux \& P.Stos~\cite[Theorem~3]{descours-guillin-lacour-michel-nectoux-stos-2026} addresses this question by extending the class of admissible test
functions. More precisely, instead of considering arbitrary test functions, we introduce a family of solutions to a nonlocal backward transport equation whose role is to compensate the drift term appearing in~\eqref{eq:weak-fluctuation-edps}. More precisely, the correction is obtained by transporting the test functions backward along the linearized mean-field dynamics. Such an approach can be found e.g. in the works by S.Bourguin \& K.Spiliopoulos~\cite{bourguin-spiliopoulos-2026} or by Z.Wang, X.Zhao \& R.Zhu~\cite[Section~3.4]{wang-zhao-zhu-2023}. This transformation allows the stochastic evolution to be expressed solely in terms of its Gaussian forcing and ultimately yields an explicit characterization of the law of the limiting fluctuation process. This leads to the following nonlocal backward transport equation.

More precisely, given a terminal time $t \geq 0$ and a final condition $\varphi \in \mc C_b^\infty(\R^d)$, we consider the following nonlocal backward transport equation.
{ \footnotesize\begin{equation}\label{eq:backward-transport}
  \begin{cases}
    \partial_s f(s,w) + \alpha\,\mbb E_\pi\Big[\big(y-\langle\sigma^*(\cdot,x),\bar\mu_s\rangle\big)\nabla f(s,w)\cdot\nabla\sigma^*(w,x)  -  \big\langle\nabla f(s,\cdot)\cdot\nabla\sigma^*(\cdot,x),\bar\mu_s\big\rangle\sigma^*(w,x)\Big]=0,
    \\
    f(t,w) = \varphi(w).
  \end{cases}
\end{equation}}

We can now state the main result of A.Descours, A.Guillin, G.Lacour, M.Michel, B.Nectoux \& P.Stos \cite[Theorem~3]{descours-guillin-lacour-michel-nectoux-stos-2026}, which provides an explicit representation of the covariance of the
limiting fluctuation process and therefore completely characterizes its
law.

\begin{thm}\label{thm:mean-field-uq} Assume that Assumptions~\ref{ass:data}--\ref{ass:batch-limit} hold and that $\beta > \frac{3}{4}$. Then, for every $n \geq 1$ and every $\varphi_1,\ldots,\varphi_n \in \mc C_b^\infty(\R^d)$, the finite-dimensional process
\[
	\Big(\langle\varphi_1,\bar\eta\rangle, \ldots, \langle\varphi_n,\bar\eta\rangle \Big) \in \mathcal C(\R_+,\R^n)
\]
is a centered Gaussian process. Its covariance is explicitly given, for every $\varphi,\psi \in \mc C_b^\infty(\R^d)$ and every $0 \leq s \leq t$, by
\begin{equation}\label{eq:covariance}
	\begin{split}
		\Cov\big(\langle\varphi,\bar\eta_s\rangle,\langle\psi,\bar\eta_t\rangle\big) = &\, \Cov\big(f^\varphi(0,W_0^1),f^\psi(0,W_0^1)\big)\\
		&+ \alpha^2\mbb E\bigg[\frac{1}{|B_\infty|}\bigg]\int_0^s\Cov_\pi\big(Q_r[f^\varphi(r)],Q_r[f^\psi(r)]\big)\,dr,
	\end{split}
\end{equation}
where
$f^\varphi$ and $f^\psi$ denote the solutions of \eqref{eq:backward-transport} on $[0,s]$ and $[0,t]$ with terminal conditions $f^\varphi(s,\cdot)=\varphi$ and $f^\psi(t,\cdot)=\psi$, respectively.
\end{thm}

An immediate consequence of Theorem~\ref{thm:mean-field-uq} is the following explicit variance representation.
\begin{equation}\label{eq:variance}
  \Var\big(\langle \varphi, \bar\eta_t\rangle\big) = \Var\big(f^\varphi(0,W_0^1)\big) + \alpha^2\mbb E\bigg[\frac{1}{|B_\infty|}\bigg]\int_0^t\Var\big(Q_s[f^\varphi(s)]\big)\,ds.
\end{equation}
Choosing $\varphi = \sigma^*(\cdot, x)$ yields an explicit variance formula for the prediction of the neural network at the input $x \in \mathcal X$. This identity provides a direct uncertainty quantification formula in the mean-field regime.

\begin{rmk}\label{rmk:sobolev-gaussianity-fluctuation-process} Since $C_c^\infty(\R^d) \subset C^\infty_b(\R^d)$ and $C^\infty_c(\R^d)$ is dense in $\msr  H' \simeq \mc H^{J_0-1,j_0}(\R^d)$, Theorem~\ref{thm:mean-field-uq} also implies that $\bar\eta_t$ is an $\msr H$-valued Gaussian random variable. More precisely, for every $\varphi \in \msr H'$, choose $\varphi_n \in C_c^\infty(\R^d)$ such that $\varphi_n \to \varphi$ in $\msr H'$. Then
\[
    \langle\varphi_n,\bar\eta_t\rangle \to \langle\varphi,\bar\eta_t\rangle,
\]
almost surely, while every variable $\langle\varphi_n,\bar\eta_t\rangle$ is Gaussian by Theorem~\ref{thm:mean-field-uq}. Since a weak limit of Gaussian distributions on $\R$ is Gaussian, the conclusion follows. Nevertheless, this does not let us extend the explicit covariance formula \eqref{eq:covariance} in such a setting, for which a deeper analysis of qualitative properties to the solution of the backward transport equation~\eqref{eq:backward-transport} is required.
\end{rmk}

\section{Proofs of the results}\label{sec:proofs}

We begin with the proof of Theorem~\ref{thm:functional-delta-method}, recalling that its inherent functional framework is given for the fluctuation space by $\msr H = \mc H^{-J_0+1,j_0}(\R^d)$, in which Theorem~\ref{thm:mean-field-clt} provides the convergence of $\eta^N$ towards the continuous limiting process $\bar\eta$. First, we apply a joint convergence result provided by Lemma~\ref{lem:joint-convergence} in Appendix~\ref{app:auxiliary-results}. 

For completeness, and for readers interested in the relation between the two functional frameworks involved, we also establish an explicit quantitative comparison between the Wasserstein and weighted Sobolev topologies. Although this estimate is not logically required by the subsequent representation argument, it provides an independent quantitative way to the required fixed-time Sobolev convergence and makes the passage between the two topologies as self contained as possible. Nevertheless, readers that are primarily interested in the Delta method argument may skip the preliminary step.

Then, we apply the Skorokhod representation theorem to the jointly convergent pair $(\mu^N,\eta^N)$ and work, on a common probability space, with representatives converging almost surely in the corresponding product Skorokhod space. Finally, we use the identity $\eta_t^N = \sqrt{N}(\mu_t^N-\bar\mu_t)$ together with the Fréchet differentiability of the observation functional. A first-order expansion around the mean field limit and a control of the remainder then give the desired convergence in distribution.

\begin{proof}[Proof of Theorem~\ref{thm:functional-delta-method}] As above mentioned, we prove the result for a fixed time $t \geq 0$. By Theorem~\ref{thm:mean-field-lln}, the empirical measure $\mu^N$ converges in probability to the deterministic path $\bar\mu$ in $\mc B_1 := \mc D(\R_+, \mc P_\gamma(\R^d))$, implying that $\mu^N \xrightarrow[N\to\infty]{\mathscr L} \bar\mu$ in the same functional space. Moreover, by Theorem~\ref{thm:mean-field-clt}, the fluctuation process $\eta^N:=\sqrt N(\mu^N-\bar\mu)$ converges in law to $\bar\eta$ in $\mc B_2 := \mc D(\R_+,\msr H)$. Then, we observe that $\mc B_1 \times \mc B_2$ equipped with the metric 
\[
	d_{\mc B_1 \times \mc B_2}\bigg((\mu_1,\eta_1),(\mu_2,\eta_2)\bigg) := \sqrt{d_{\mc B_1}(\mu_1,\mu_2)^2 + d_{\mc B_2}(\eta_1,\eta_2)^2}
\]
is a Polish space, see e.g. results in the book by C.Villani~\cite[Theorem 6.18]{villani-2009} and in the book by J.Jacod \& A.Shiryaev~\cite[Chapter VI, Theorem 1.14-a]{jacod-shiryaev-2003}, where $d_{\mc B_1}$ and  $d_{\mc B_2}$ stand respectively for the usual ($J_1$) Skorokhod topologies on $\mc P_\gamma(\R^d)$ endowed with Wasserstein distance and on $\mathscr H$ for the metric introduced in Section~\ref{sec:preliminary-results}. Applying now Lemma~\ref{lem:joint-convergence} in Appendix~\ref{app:auxiliary-results}, we get that 
\[
	(\mu^N,\eta^N) \xrightarrow[N\to\infty]{\msr L} (\bar\mu,\bar\eta) \quad\text{in} \;\mc B_1\times\mc B_2.	
\]

As above mentioned, we begin by establishing a quantitative estimate showing how random processes are embedded from Wasserstein to negative-order weighted Sobolev topologies.

\textit{Preliminary step - A quantitative comparison between Wasserstein and weighted Sobolev topologies.} 
Let $t \geq 0$. The purpose of this step is to prove \eqref{eq:control sobolev wasserstein probab}, and consequently that $\mu_t^N\xrightarrow[N \to +\infty]{\P}\bar\mu_t$ in $\mathscr H$. We give a complete proof and a quantitative rate of convergence  with respect to the $W_\gamma$ distance. Note however that one can invoke simply that $\mathcal P_\gamma(\R^d)$ is continuously embedded in $\msr H$ so one can skip this step at the first reading and go directly to the next step (see Remark~\ref{rk:topologies} for further details). 
We introduce the intermediate space $\msr K:=\mc H^{-L,\gamma^*}(\R^d)$. Since
\[
    \mc H^{J_0-1,j_0}(\R^d) \hookrightarrow \mc H^{L,\gamma}(\R^d) \hookrightarrow \mc H^{L,\gamma^*}(\R^d),
\]
duality yields
\[
    \msr K \hookrightarrow \mc H^{-L,\gamma}(\R^d) \hookrightarrow \msr H.
\]
Let $\mu,\nu \in \mc P_{2\gamma^*}(\R^d)$ and let $\pi \in \Pi(\mu,\nu)$ denote the optimal coupling between $\mu$ and $\nu$ fo $W_2$. For every $\varphi \in \mc H^{L,\gamma^*}(\R^d)$, we can write
\[
	\big|\varphi(w) - \varphi(w')\big| \leq \sup_{s\in[0,1]} \big|\nabla\varphi(w + s(w'-w))\big||w-w'| \leq C\|\varphi\|_{\mc H^{L,\gamma^*}}|w - w'|\big(1 + |w|^{\gamma^*} + |w'|^{\gamma^*}\big),
\]
for some $C>0$. Integrating with respect to $\pi$, we obtain 
\[
	\big|\langle \varphi, \mu - \nu\rangle\big| \leq  C\|\varphi\|_{\mc H^{L,\gamma^*}}\int_{\R^d \times \R^d}|w - w'|\big(1 + |w|^{\gamma^*} + |w'|^{\gamma^*}\big)\, \di\pi(w,w'),
\]
and Hölder inequality let us conclude, for a different constant that we still denote by $C$, 
\[
	\big|\langle \varphi, \mu - \nu\rangle\big| \leq  C\|\varphi\|_{\mc H^{L,\gamma^*}}\bigg(1 + \int_{\R^d \times \R^d}|w|^{2\gamma^*} + |w'|^{2\gamma^*} \, \di\pi(w,w')\bigg)^{\frac{1}{2}} W_2(\mu,\nu).
\]
Since $\gamma \geq 2$, $W_2(\mu,\nu) \leq W_\gamma(\mu,\nu)$ (see e.g. the book by C.Villani~\cite[Remark 6.6]{villani-2009}), so that one gets
\[
	\big|\langle \varphi, \mu - \nu \rangle\big| \leq C\bigg(1 + \int_{\R^d \times \R^d}|w|^{2\gamma^*} + |w'|^{2\gamma^*} \, \di\pi(w,w')\bigg)^{\frac{1}{2}}\|\varphi\|_{\mc H^{L,\gamma^*}}W_\gamma(\mu,\nu),
\]	
and therefore, taking the supremum on both sides over the closed unit ball in $\mc H^{L,\gamma^*}(\R^d)$ yields
\begin{equation}\label{eq:embedding-wasserstein-dual-sobolev}
	\|\mu - \nu \|_{\msr K} \leq C\bigg(1 + \int_{\R^d \times \R^d}|w|^{2\gamma^*} + |w'|^{2\gamma^*} \, \di\pi(w,w')\bigg)^{\frac{1}{2}}W_\gamma(\mu,\nu).
\end{equation}
Our goal is now to apply this inequality to $\mu_t^N$ and $\bar\mu_t$.  
We first point out that $\mu_t^N$ does satisfy the above required higher-order moment control, thanks to the result by A.Descours, A.Guillin, M.Michel \& B.Nectoux~\cite[Lemma~8]{descours-guillin-michel-nectoux-2024}. Namely, for every $T>0$, there exists $C_T>0$ such that
\[
	\underset{N \geq 1}{\mathrm{sup}}\, \underset{1 \leq i \leq N}{\mathrm{sup}}\, \underset{0 \leq k \leq \lfloor NT \rfloor}{\mathrm{sup}}\,\mbb E\big[|W_k^{i,N}|^{8\gamma^*}\big]\leq C_T.
\]
Hence, for every fixed $t\in[0,T]$, this immediately implies
\[
\underset{N \geq 1}{\mathrm{sup}}\, \mbb E\bigg[\int_{\R^d}|w|^{2\gamma^*}\di\mu_t^N(w)\bigg] < +\infty.
\]

Therefore, Markov inequality gives, for every $ R > 0$,
\begin{equation}\label{markov mu t N}
	\underset{N \geq 1}{\mathrm{sup}}\, \P\bigg(\int_{\R^d}|w|^{2\gamma^*}\,\di\mu_t^N(w) > R\bigg) \leq \frac{1}{R}\underset{N \geq 1}{\mathrm{sup}}\,\mbb E\bigg[\int_{\R^d}|w|^{2\gamma^*}\,\di\mu_t^N(w)\bigg].
\end{equation}

We now prove that $\dst\int_{\R^d}|w|^{2\gamma^*}\,\di\bar\mu_t(w) < +\infty$. Indeed, by Corollary~\ref{coro:Wgamma-continuity}, $\mu_t^N \to \bar\mu_t$ in probability in $\mc P_\gamma(\R^d)$, and therefore also in distribution. 

Since the function $w \mapsto  |w|^{2\gamma^*}$ is nonnegative and lower semicontinuous, we have, see e.g. in the book by L.Ambrosio, N.Gigli \& G.Savaré~\cite[Lemma~5.1.7]{ambrosio-gigli-savare-2008}, that the extended valued moment functional from $\mc P_\gamma(\R^d)$ into $\R_+$ given by $\mu \mapsto \dst\int_{\R^d}|w|^{2\gamma^*}\,\di\mu(w)$ is lower semicontinuous for the narrow topology. Combining this lower semicontinuity property together with the same argument as the one used to prove \eqref{eq:Fatou+Ambrosio} that
\begin{equation}\label{mu bar bounded moments gamma 2}
    \int_{\R^d}|w|^{2\gamma^*}\,\di\bar\mu_t(w) \leq \underset{N \to +\infty}{\mathrm{lim}\,\mathrm{inf}}\,\mbb E\bigg[\int_{\R^d}|w|^{2\gamma^*}\,\di\mu_t^N(w)\bigg] \leq \underset{N \geq 1}{\mathrm{sup}}\,\mbb E\bigg[\int_{\R^d}|w|^{2\gamma^*}\,\di\mu_t^N(w)\bigg] < +\infty.
\end{equation}
Together with the preceding uniform expectation estimate, this shows that $\mu_t^N$ and $\bar\mu_t$ define elements of $\msr K$, almost surely. Now, for $R > 0$, set
\[
    A_{N,R} := \bigg\{\int_{\R^d}|w|^{2\gamma^*}\,\di\mu_t^N(w) \leq R \bigg\}.
\]
For every $\ve > 0$, estimate~\eqref{markov mu t N} allows us to choose $R_\ve > 0$ such that
\[
    \underset{N \geq 1}{\mathrm{sup}}\,\P\big(A_{N,R_\ve}^{c}\big) \leq \ve.
\]
On the event $A_{N,R_\ve}$, applying \eqref{eq:embedding-wasserstein-dual-sobolev} with $\mu = \mu_t^N$ and $\nu = \bar\mu_t$ gives the quantitative estimate
\begin{equation}\label{eq:control sobolev wasserstein probab}
    \|\mu_t^N-\bar\mu_t\|_{\msr K} \leq C\bigg( 1 + \int_{\R^d}|w|^{2\gamma^*}\,\di\bar\mu_t(w) +R_\ve \bigg)^{\frac{1}{2}}W_\gamma(\mu_t^N,\bar\mu_t).
\end{equation}
Set
\[
    C_{\ve,t} := C\bigg( 1 + \int_{\R^d}|w|^{2\gamma^*}\,\di\bar\mu_t(w) + R_\ve \bigg)^{\frac{1}{2}}.
\]
Then, for every $\de > 0$,
\[
\begin{split}
    \P\big(\|\mu_t^N-\bar\mu_t\|_{\msr K} > \de \big)
    &\leq \P\big(A_{N,R_\ve}^{\,c}\big) + \P\big(C_{\ve,t}W_\gamma(\mu_t^N,\bar\mu_t) > \de\big)\\
    &\leq \ve + \P\bigg(W_\gamma(\mu_t^N,\bar\mu_t)>\frac{\de}{C_{\ve,t}}\bigg).
\end{split}
\]
Since $\mu_t^N\xrightarrow[N \to +\infty]{\P}\bar\mu_t$ in $\mc P_\gamma(\R^d)$ by Corollary~\ref{coro:Wgamma-continuity}, we obtain
\[
    \underset{N \to +\infty}{\mathrm{lim}\, \mathrm{sup}}\,\P\big(\|\mu_t^N-\bar\mu_t\|_{\msr K} > \de\big) \leq \ve.
\]
Letting $\ve \to 0$, we conclude that $\mu_t^N \xrightarrow[N \to +\infty]{\P}\bar\mu_t$ in $\msr K$ and hence in $\msr H$ by the continuous embedding $\msr K \hookrightarrow \msr H$.

\textit{First step - Representations in product space.} We now use the Skorokhod representation theorem, see for instance the books by P.Billingsley~\cite[Theorem 6.7]{billingsley-1999} or by V.Bogachev~\cite[Theorem 8.5.4]{bogachev-2007}. It allows us to construct, on a common probability space, random variables
\[
    (\mathsf M^N,\mathsf E^N) \overset{\mathscr L}{=} (\mu^N,\eta^N)
\]
and a limiting random variable
\[
    (\bar\mu,\bar{\mathsf E}) \overset{\mathscr L}{=} (\bar\mu,\bar\eta)
\]
such that
\begin{equation}\label{eq:product-law-convergence}
    (\mathsf M^N,\mathsf E^N) \to (\bar\mu,\bar{\mathsf E}) \quad \text{a.s.}\; \text{in}\; \mc B_1 \times \mc B_2.
\end{equation}

Now for every fixed $N \geq 1$ and $t \geq 0$, consider the Borel measurable mapping $G_{N,t}(m,e) := e_t - \sqrt{N}\big(m_t-\bar\mu_t\big)$, with values in $\msr H$. Also  since $(\msf M^N,\msf E^N) \overset{\mathscr L}{=} (\mu^N,\eta^N)$, and $G_{N,t}$ is measurable, and since equality in law is preserved under measurable mapping, then
\[
    G_{N,t}(\msf M^N,\msf E^N) \overset{\mathscr L}{=} G_{N,t}(\mu^N,\eta^N).
\]
Now by the definition of the fluctuation process,
\[
    G_{N,t}(\mu^N,\eta^N) = \eta_t^N - \sqrt{N}\big(\mu_t^N - \bar\mu_t\big) = 0,
\]
almost surely in $\msr H$. Therefore,
\[
    \msr L \big(G_{N,t}(\msf M^N,\msf E^N)\big) = \delta_0,
\]
and hence,
\begin{equation}\label{eq:identity-skorokhod-representation}
    \msf E_t^N = \sqrt{N}\big(\msf M_t^N - \bar\mu_t\big),
\end{equation}
almost surely in $\msr H$. Moreover, since the limiting path $\bar{\msf E}$ is continuous, the almost sure convergence in \eqref{eq:product-law-convergence} implies
\[
    \msf E_t^N \to \bar{\msf E}_t
\]
almost surely in $\msr H$. In particular, \eqref{eq:identity-skorokhod-representation} already gives that almost surely in $\msr H$
\[
    \msf M_t^N - \bar\mu_t  = N^{-1/2}\msf E_t^N \to 0.
\]

\textit{Second step - Functional Delta method argument.} 
By assumption, the observation functional $\Phi : \msr H \to \R^k$ is Fréchet differentiable at $\bar\mu_t$, so that there exists a remainder map $r_t: \msr H \to \R^k$ such that

\[
    \Phi(\bar\mu_t+h) = \Phi(\bar\mu_t) + D\Phi(\bar\mu_t)(h) + r_t(h),
\]
with
\[
    \frac{|r_t(h)|_{\R^k}}{\|h\|_{\msr H}} \underset{h \to 0}{\longrightarrow} 0
\]

in $\msr H$. Using this expansion with $h = \msf M_t^N - \bar\mu_t = N^{-\frac{1}{2}}\msf E_t^N$, we obtain
\[
	\sqrt{N} \big(\Phi(\msf M_t^N) - \Phi(\bar\mu_t)\big) = D\Phi(\bar\mu_t)(\msf E_t^N) + \sqrt{N}\,r_t(\msf M_t^N - \bar\mu_t).
\]
We now show that the remainder converges to zero almost surely. With the convention that the following quotient is zero whenever $\msf M_t^N = \bar\mu_t$, we have
\[
	\big|\sqrt{N}\,r_t(\msf M_t^N - \bar\mu_t)\big|_{\R^k} = \|\msf E_t^N\|_{\msr H} \frac{|r_t(\msf M_t^N - \bar\mu_t)|_{\R^k}}{\|\msf M_t^N - \bar\mu_t\|_{\msr H}}.
\]
Since $\msf M_t^N  \xrightarrow[N\to\infty]{a.s.} \bar\mu_t$ in $\msr H$, the definition of Fréchet differentiability gives
\[
	\frac{|r_t(\msf M_t^N - \bar\mu_t)|_{\R^k}}{\|\msf M_t^N - \bar\mu_t\|_{\msr H}} \xrightarrow[N \to +\infty]{a.s.} 0.
\]
Also, $\msf E_t^N \to \bar{\msf E}_t$ almost surely in $\msr H$ so that $(\|\msf E_t^N\|_{\msr H})_{N \geq 1}$ is bounded almost surely. It follows that the remainder satisfies in $\R^k$
\[
    \sqrt{N}\,r_t(\msf M_t^N - \bar\mu_t) \xrightarrow[N\to\infty]{a.s.} 0.
\]

Moreover, since $D\Phi(\bar\mu_t) : \msr H \to \R^k$ is linear and continuous, we get that $D\Phi(\bar\mu_t)(\msf E_t^N) \to D\Phi(\bar\mu_t)(\bar{\msf E}_t)$ almost surely in $\R^k$. Thus,

\[
	\sqrt{N}\Big(\Phi(\msf M_t^N) - \Phi(\bar\mu_t)\Big) \xrightarrow[N \to +\infty]{a.s.} D\Phi(\bar\mu_t)(\bar{\msf E}_t)
\]
in $\R^k$. Finally, since $(\msf M^N,\msf E^N) \overset{\mathscr L}{=} (\mu^N,\eta^N)$ and $\bar{\msf E} \overset{\mathscr L}{=}\bar\eta$, the almost sure convergence obtained on the Skorokhod representation space implies the desired convergence in distribution, namely
\[
	\sqrt{N}\Big(\Phi(\mu_t^N) - \Phi(\bar\mu_t)\Big) \xrightarrow[N \to +\infty]{\mathscr L} D\Phi(\bar\mu_t)(\bar\eta_t)
\]
in $\R^k$. Since $\bar\eta_t$ is Gaussian and $D\Phi(\bar\mu_t)$ is linear and continuous, the limiting random vector $D\Phi(\bar\mu_t)(\bar\eta_t)$ is Gaussian in $\R^k$ (see item 1. in Definition~\ref{defi:g-process}). This achieves the proof.
\end{proof}

We now prove the corollary, which is a straightforward consequence of Theorem~\ref{thm:mean-field-uq} and of Theorem~\ref{thm:functional-delta-method}. 

\begin{proof}[Proof of Corollary~\ref{coro:covariance-formula}] By Theorem~\ref{thm:functional-delta-method}, the limiting random vector is
\[
    Z_t^\Phi = \Big(D\Phi_1(\bar\mu_t)(\bar\eta_t), \ldots, D\Phi_k(\bar\mu_t)(\bar\eta_t)\Big).
\]
Using the representation assumption, we can write it
\[
    Z_t^\Phi = \Big(\langle \varphi_1^t,\bar\eta_t\rangle, \ldots, \langle \varphi_k^t,\bar\eta_t\rangle \Big).
\]
Hence
\[
    \Sigma_{i,j}^\Phi(t) = \Cov\Big(\langle \varphi_i^t,\bar\eta_t\rangle, \langle \varphi_j^t,\bar\eta_t\rangle \Big).
\]
Applying Theorem~\ref{thm:mean-field-uq} with $s=t$, $\varphi=\varphi_i^t$ and $\psi=\varphi_j^t$ gives
\[
    \Sigma_{i,j}^\Phi(t) = \Cov\Big(f^{\varphi_i^t}(0,W_0^1), f^{\varphi_j^t}(0,W_0^1)\Big) + \alpha^2\mbb E\bigg[\frac{1}{|B_\infty|}\bigg]\int_0^t\Cov_\pi\Big(Q_r[f^{\varphi_i^t}(r)], Q_r[f^{\varphi_j^t}(r)]\Big)\,\di r.
\]
This proves the claim.
\end{proof}

\begin{rmk}\label{rk:topologies}
The argument used above was deliberately formulated in terms of a simple functional structure that is sufficient to apply the functional delta method. That being said, a more intrinsic connection between the Wasserstein and weighted Sobolev spaces can be established. Such a continuity property is not automatic for arbitrary combinations of Wasserstein and dual Hölder topologies. For the polynomial growth exponent considered here, however, the canonical embedding $i : \big(\mc P_\gamma(\R^d),W_\gamma\big) \to \big(\big(C^{J,\gamma}(\R^d)\big)^*,\|\cdot\|_{(C^{J,\gamma}(\R^d))^*}\big)$, $ \iota(\mu)(f):=\int_{\R^d}f\,\di\mu$, is well defined and continuous for $J \geq 1$. Combined with the continuous dual Sobolev embedding
\[
    \big(C^{J,\gamma}(\R^d)\big)^* \hookrightarrow \msr H
\]
used in the present paper, which holds under well-chosen parameters $J_0, j_0$ , this also yields the continuity of the canonical embedding of $\big(\mc P_\gamma(\R^d),W_\gamma\big)$ into $\big(\msr H,\|\cdot\|_{\msr H}\big)$. Let us briefly explain the first continuity statement. Suppose that $\mu^N \to \mu$ in $W_\gamma$, and let $\chi_R$ be a smooth cutoff function equals to one on a ball $B_R$, vanishing outside $B_{R+1}$, and taking values in $[0,1]$. For every $f \in C^{J,\gamma}(\R^d)$, we then can write $f = \chi_Rf+( 1-\chi_R)f$. If $\|f\|_{C^{J,\gamma}}\leq 1$, then, for every fixed $R>0$, the bounded Lipschitz norm of $\chi_Rf$ is bounded by a constant $C_R$ independent of $f$. Consequently, denoting by $d_{\mathrm{BL}}$ the bounded Lipschitz distance,
\[
    \underset{\|f\|_{C^{J,\gamma}}\leq 1}{\mathrm{sup}}\;\bigg|\int_{\R^d}\chi_Rf\,\di(\mu^N-\mu)\bigg| \leq C_R\, d_{\mathrm{BL}}(\mu^N,\mu).
\]
The other term is controlled by
\[
\begin{split}
    \underset{\|f\|_{C^{J,\gamma}}\leq 1}{\mathrm{sup}}\;\bigg|\int_{\R^d}(1-\chi_R)f\,\di(\mu^N-\mu)\bigg| \leq \int_{\{|x|>R\}}(1+|x|^\gamma)\,\di\mu^N(x) + \int_{\{|x|>R\}}(1+|x|^\gamma)\,\di\mu(x).
\end{split}
\]
Since convergence in $W_\gamma$ implies weak convergence together with uniform integrability of the moments of order $\gamma$, the right-hand side in the above estimate is uniformly small as $R \to +\infty$. Then, for every fixed $R$, weak convergence gives $d_{\mathrm{BL}}(\mu^N,\mu) \to 0$. Taking first the upper limit as $N \to +\infty$ and then letting $R \to +\infty$ gives
\[
    \|i(\mu^N)-i(\mu)\|_{(C^{J,\gamma}(\R^d))^*} \to 0.
\]

The above observation provides an alternative way to prove Theorem~\ref{thm:functional-delta-method}. Indeed, after applying the
Skorokhod representation theorem to the joint convergence, one obtains
representatives such that $\msf M^N \to \bar\mu$ almost surely in $\mc D\big(\R_+,\mc P_\gamma(\R^d)\big)$. Since the limiting trajectory $\bar\mu$ is continuous for the $W_\gamma$ topology, it follows at every fixed time $t \geq 0$ that $W_\gamma(\msf M_t^N,\bar\mu_t) \to 0$ almost surely. The above continuous embedding then yields $\msf M_t^N \to \bar\mu_t$ almost surely in $\msr H$, and the Taylor expansion appearing in the functional delta
method part of the proof can be applied directly.
\end{rmk}

The representation assumption in Corollary~\ref{coro:covariance-formula} deserves some additional discussion. Since $H$ is a Hilbert space, every continuous linear functional $D\Phi_i(\bar\mu_t) \in \msr H'$ is represented by a unique element of $\msr H' \simeq \mc H^{J_0-1,j_0}(\R^d)$. The additional assumption in Corollary~\ref{coro:covariance-formula} is therefore not the existence of such a Sobolev representative, but its stronger regularity $\varphi_i^t\in C_b^\infty(\R^d)$, which allows Theorem~\ref{thm:mean-field-uq} to be applied directly. In practice, such a representation arises naturally for a broad class of observables. For instance, if $\Phi_i(\mu)=\langle q_i,\mu\rangle$, then $D\Phi_i(\mu)(h)=\langle q_i,h\rangle$, so that one may simply take $\varphi_i^t=q_i$. More generally, for a cylindrical observable of the form
\[
	\Phi_i(\mu) = F_i\big(\langle q_1,\mu\rangle,\ldots,\langle q_m,\mu\rangle\big),
\]
the chain rule yields
\[
	D\Phi_i(\bar\mu_t)(h) = \bigg\langle\sum_{\ell=1}^m\partial_\ell F_i\big(\langle q_1,\bar\mu_t\rangle,\ldots,\langle q_m,\bar\mu_t\rangle\big)q_\ell, h \bigg\rangle.
\]
Hence, in this case, the representing function is explicitly given by
\[
	\varphi_i^t = \sum_{\ell=1}^m\partial_\ell F_i\big(\langle q_1,\bar\mu_t\rangle,\ldots,\langle q_m,\bar\mu_t\rangle\big)q_\ell.
\]
In particular, the regularity assumption of Corollary~\ref{coro:covariance-formula} is satisfied whenever the functions $q_\ell$ belong to $C_b^\infty(\mathbb R^d)$. 

Also and as aforementioned, the smoothness requirement on the representing functions is mainly inherited from the covariance representation established in Theorem~\ref{thm:mean-field-uq}. One may in principle seek to extend this representation to a larger Sobolev class by establishing suitable continuity properties of the solution to the associated backward transport equation with respect to its terminal datum. Such an extension appears natural in view of the continuity properties of the backward solution operator established in the previous work, but would require a slight refinement of those arguments. We therefore restrict here to representatives for which Theorem~\ref{thm:mean-field-uq} applies directly.

We conclude this section by proving the local factorization criterion for identifiability. The result shows that, under the constant rank assumption, the differential kernel condition throughout a neighborhood integrates into an exact local nonlinear factorization.

\begin{proof}[Proof of Theorem~\ref{thm:local-identifiability}]
We first prove that $(i)$ implies $(ii)$. Assume that there exist open neighborhoods $\mc U_1 \subset \mc U_0$ and $\mc V$, together with a continuously differentiable mapping $F :\mc V \to \R^\ell$, such that it holds on $\mc U_1$ that $Q = F \circ \Phi$. The chain rule formula gives, for every $\mu \in \mc U_1$ that $DQ(\mu) = DF(\Phi(\mu)) \circ D\Phi(\mu)$. Therefore, if $h \in \mathrm{Ker}\, D\Phi(\mu)$, then
\[
    DQ(\mu)(h) = DF(\Phi(\mu))\big(D\Phi(\mu)(h)\big) = DF(\Phi(\mu))(0) = 0.
\]
Hence for every $\mu \in \mc U_1$ we get $\mathrm{Ker}\, D\Phi(\mu) \subset \mathrm{Ker}\, DQ(\mu)$, which proves $(ii)$.

We now prove that $(ii)$ implies $(i)$. Set $ K:= \mathrm{Ker}\, D\Phi(\mu_0)$ and $E := K^\perp$, so that $\msr H = E \oplus K$ and the restriction $D\Phi(\mu_0)_{|E} :E \to \mathrm{Ran}\,D\Phi(\mu_0)$ is an isomorphism. In particular, $\dim E=r$. Since $E$ is finite-dimensional and $D\Phi$ is continuous, up to reduce $\mc U_2$, $D\Phi(\mu)_{|E}$ remains one to one for every $\mu \in \mc U_2$. The constant rank assumption then yields $D\Phi(\mu)(E) = \mathrm{Ran}\,D\Phi(\mu)$, so that $D\Phi(\mu)_{|E} : E \to \mathrm{Ran}\,D\Phi(\mu)$ is an isomorphism. Hence the assumptions of the constant rank theorem (see e.g. R.Abraham, J.Marsden \& T.Ratiu~\cite[Theorem 2.5.15]{abraham-masden-ratiu-1988}) are satisfied, and it provides open neighborhoods $\mc O_H \subset \mc U_2$ of $\mu_0$ in $H$, $\mc O_{\R^k}$ of $\Phi(\mu_0)$ in $\R^k$,and $C^1$ coordinate maps $\Psi:\mc O_H \to \mc O_E \times \mc O_K$ and $\xi: \mc O_{\R^k} \to \mc O_{\R^r} \times \mc O_{\R^{k-r}}$ which are $C^1$ diffeomorphisms onto their images and satisfy  $\Psi(\mu_0)=(0,0)$, $\xi(\Phi(\mu_0))=(0,0)$, where $\mc O_E \subset E$, $\mc O_K \subset K$, $O_{\R^r} \subset \R^r$, $O_{\R^{k-r}} \subset \R^{k-r}$ are open neighborhoods of $0$ respectively in each space. Also, after possibly reducing these neighborhoods, we may assume that $\mc O_K$ is convex and that we have
\[
    \big(\xi \circ \Phi \circ \Psi^{-1}\big)(u,v) = (u,0)
\]
for every $(u,v) \in \mc O_E \times \mc O_K$, after identifying $E$ with $\R^r$. Define the mapping $ q : \mc O_E \times \mc O_K \to \R^\ell$, $(u,v) \mapsto Q(\Psi^{-1}(u,v))$. Then, for every $(u,v)$ in this product neighborhood and every $\eta \in K$, we have
\[
    D\big(\xi \circ \Phi \circ \Psi^{-1}\big)(u,v)(0,\eta) = 0.
\]
Since $\xi$ is a $C^1$ dffeomorphism, we have that $D\xi\big(\Phi\big(\Psi^{-1}(u,v)\big)\big)$ is one to one, so it follows that $D\Psi^{-1}(u,v)(0,\eta) \in \mathrm{Ker}\, D\Phi(\Psi^{-1}(u,v))$. Using the kernel inclusion in $(ii)$, we thus obtain
\[
    DQ(\Psi^{-1}(u,v))\big(D\Psi^{-1}(u,v)(0,\eta)\big) = 0.
\]
Thanks to the chain rule formula, this is equivalent to $\partial_vq(u,v)(\eta) = 0$. Therefore, for every fixed $u$ and every $v_1,v_2 \in \mc O_K$, it holds
\[
	q(u,v_2)-q(u,v_1) = \int_0^1 \partial_vq\big(u,v_1+s(v_2-v_1)\big)(v_2-v_1)\,\di s  = 0.
\]
Hence $q(u,v)$ does not depend on $v$. Defining then $\widehat F: \mc O_E \to \R^\ell$, $u \mapsto q(u,0)$, we thus get $q(u,v) = \widehat F(u)$. 

Let $\pi_r: \R^r \times \R^{k-r} \to \R^r$ denote the projection onto the first component. Since $\pi_r\xi(\Phi(\mu_0)) = 0 \in \mc O_E$, by continuity of $\pi_r \circ \xi$ we may choose an open neighborhood $\mc V \subset \mc O_{\R^k}$ of $\Phi(\mu_0)$ such that $\pi_r\xi(\mc V)\subset\mc O_E$. We may therefore define $F: \mc V \to \R^\ell$, $y \mapsto \widehat F\big(\pi_r\xi(y)\big)$. Let us point out that $F$ is continuously differentiable since $\widehat F$, $\pi_r$, and $\xi$ are so. Since $\Phi$ is continuous and $\Phi(\mu_0) \in \mc V$, we may further choose an open neighborhood $\mc U_1 \subset \mc O_H \subset \mc U_2$ of $\mu_0$ such that $\Phi(\mc U_1)\subset\mc V$. Hence, for every such $\mu$, we get that
\[
    F(\Phi(\mu)) = \widehat F\big(\pi_r\xi(\Phi(\mu))\big) = \widehat F(u) = q(u,v) = Q(\mu).
\]
Hence $Q= F \circ \Phi$ on such a defined open set $\mc U_1$, with $\mc U_1 \subset \mc U_0$, $\Phi(\mc U_1) \subset \mc V$, and $F \in C^1(\mc V,\R^\ell)$, which proves $(i)$.

In particular, if $\mu_1,\mu_2 \in \mc U_1$ satisfy $\Phi(\mu_1)=\Phi(\mu_2)$, then $Q(\mu_1) = F(\Phi(\mu_1)) = F(\Phi(\mu_2)) = Q(\mu_2)$.
\end{proof}

\appendix

\section{Auxiliary probabilistic and functional results}\label{app:auxiliary-results} 

As discussed throughout the article, some arguments rely on auxiliary technical results of rather different nature, whose detailed proofs would interrupt the clarity of the presentation. We have therefore gathered them in the present appendix. First, we establish the probabilistic argument that combines the convergence of the empirical measure with that of its fluctuation process. We then turn to a specific property of the weighted Sobolev scale underlying our functional framework, namely the Hilbert-Schmidt character of the embeddings used throughout the paper.

\subsection{A joint convergence lemma}

The following lemma isolates the joint convergence argument used in the proof of Theorem~\ref{thm:functional-delta-method}. In the present setting, the result may be viewed as a kind of generalized Slutsky theorem, or obtained from a usual converging-together result for random elements taking values in Polish spaces. Nevertheless, we provide thereafter a direct proof based on bounded Lipschitz test functions. Besides making the paper self-contained, this formulation makes explicit the underlying probabilistic mechanism used in the argument. More preciselty, it highlights the fact that convergence in probability of the empirical measure towards a deterministic limit can be combined with convergence in distribution of the fluctuation process to obtain their joint convergence.

\begin{lemma}[Joint convergence]\label{lem:joint-convergence}
Assume that the assumptions of Theorems~\ref{thm:mean-field-lln} and~\ref{thm:mean-field-clt} hold. Set $\mc B_1 := \mc D\big(\R_+,\mc P_\gamma(\R^d)\big)$ and $\mc B_2 := \mc D\big(\R_+,\msr H\big)$, both endowed with their usual Skorokhod $J_1$ topologies. Then
\[
    (\mu^N,\eta^N) \xrightarrow[N\to\infty]{\msr L} (\bar\mu,\bar\eta) \quad\text{in} \;\mc B_1\times\mc B_2.
\]
\end{lemma}

\begin{proof}
	The space $\mc P_\gamma(\R^d)$ endowed with the Wasserstein distance $W_\gamma$ is Polish. Moreover, the Skorokhod space $\mc D(\R_+,E)$ endowed with the $J_1$ topology is Polish whenever $E$ is Polish. Consequently, $\mc B_1$, $\mc B_2$, and $\mc B_1\times\mc B_2$ are Polish spaces. We then set $\beta^N := (\mu^N,\eta^N)$, $\tilde\beta^N := (\bar\mu,\eta^N)$ and $\beta := (\bar\mu,\bar\eta)$. Therefore, we get $d_{\mc B_1 \times \mc B_2}(\beta^N, \tilde \beta^N) = d_{\mc B_1}(\mu^N,\bar\mu)$ so that $d_{\mc B_1 \times \mc B_2}(\beta^N, \tilde\beta^N) \to 0$ in probability, since Theorem~\ref{thm:mean-field-lln} gives $\mu^N \to \bar\mu$ in probability in $\mc B_1$.

We also set the mapping $F_{\bar\mu} : \mc B_2 \to \mc B_1 \times \mc B_2$ defined by $F_{\bar\mu}(\eta) = (\bar\mu,\eta)$, which is clearly continuous for the considered topologies. Then, we point out that thanks to Theorem~\ref{thm:mean-field-clt} and the continuous mapping theorem (see e.g. the book by S.Resnick~\cite[Corollary 8.3.1]{resnick-2005}), we get that 
\begin{equation}\label{eq:convergence-law-tilde-beta}
	 \tilde\beta^N = F_{\bar\mu}(\eta^N) \xrightarrow[N\to\infty]{\mathscr L} F_{\bar\mu}(\bar\eta) = \beta \quad \text{in } \mc B_1 \times \mc B_2.
\end{equation}
Setting now $Q_N := \mc L(\beta^N)$, $\tilde Q_N := \mc L(\tilde\beta^N)$ and $Q := \mc L(\beta)$, we aim to show that
\begin{equation*}
	\langle f, Q_N \rangle \underset{N \to +\infty}{\rightarrow} \langle f, Q\rangle,
\end{equation*}
where $\langle \cdot, \cdot \rangle$ denotes the duality bracket between test functions and measures. Thanks to the Portmanteau theorem (see e.g. the book by S.Resnick~\cite[Theorem 8.4.1]{resnick-2005}) combined with the metrizability characterization of convergence in law (see the book by A.W.van der Vaart \& J.Wellner~\cite[Theorem 1.12.4]{vandervaart-wellner-2023}), it is enough to show the result for test functions $f$ belonging to the space of bounded Lipschitz functions. Assuming that such a property holds, we write
\begin{equation}\label{eq:decomposition-convergence-law-product}
	\big| \langle f, Q_N \rangle - \langle f, Q \rangle  \big| \leq \big| \langle f, Q_N \rangle - \langle f, \tilde Q_N \rangle  \big| + \big| \langle f, \tilde Q_N \rangle - \langle f, Q \rangle  \big|,
\end{equation}
and thanks to \eqref{eq:convergence-law-tilde-beta}, we already get that
\[
	\big| \langle f, \tilde Q_N \rangle - \langle f, Q \rangle  \big| \underset{N \to +\infty}{\to} 0. 
\]
We then write
\[
	\big| \langle f, Q_N \rangle - \langle f, \tilde Q_N \rangle  \big| = \bigg| \mbb E\big[f(\beta^N)\big] - \mbb E\big[ f(\tilde\beta^N)\big]\bigg| \leq \mbb E\bigg[\big| f(\beta^N) - f(\tilde\beta^N)\big|\bigg],
\]
and set $A_N^\ve := \big\{d_{\mc B_1 \times \mc B_2}(\beta^N,\tilde\beta^N) \leq \ve\big\}$. Since $f$ is assumed to be a bounded Lipschitz test function it follows
\[
	\big| \langle f, Q_N \rangle - \langle f, \tilde Q_N \rangle \big| \leq L_f\ve + 2M_f\P\big(d_{\mc B_1 \times \mc B_2}(\beta^N, \tilde\beta^N) > \ve\big),
\]
where $L_f > 0$ and $M_f > 0$ are respectively the Lipschitz and boundedness constants associated to $f$. We recall that the last term vanishes as $N \to +\infty$ since we benefit from the convergence of probability of $\mu^N$ toward $\bar\mu$ in $\mc B_1$ by Theorem~\ref{thm:mean-field-lln}, so that coming back to \eqref{eq:decomposition-convergence-law-product} we get
\[
	\big| \langle f, Q_N \rangle - \langle f, Q \rangle  \big| \leq \big| \langle f, \tilde Q_N \rangle - \langle f, Q \rangle\big| + L_f\ve + 2M_f\P\big(d_{\mc B_1 \times \mc B_2}(\beta^N, \tilde\beta^N) > \ve\big) \underset{N \to +\infty}{\to} L_f\ve,
\]

so that 
\[
    \underset{N \to +\infty}{\mathrm{lim}\, \mathrm{sup}}\, \big|\langle f,Q_N \rangle - \langle f,Q \rangle\big| \leq L_f\varepsilon.
\]
and letting $\ve \to 0$ this imply the following joint convergence
\[
    (\mu^N,\eta^N) \xrightarrow[N\to\infty]{\mathscr L} (\bar\mu,\bar\eta), \quad \text{in}\; \mc B_1 \times \mc B_2.
\]
\end{proof}

\subsection{Hilbert-Schmidt embeddings for weighted Sobolev spaces}

We now turn to the functional framework used in our results. As discussed in Section~\ref{ssec:functional-framework}, under suitable moment and regularity assumptions, probability measures can be identified with continuous linear functionals on weighted spaces of test functions and, through the corresponding continuous embeddings, with elements of dual of weighted Sobolev spaces. This provides a common space in which both the empirical measures and their fluctuation processes can be considered. 

The use of weighted Sobolev spaces is particularly convenient because they retain a Hilbert space structure while accounting for the behavior of functions and measures at infinity. This structure makes several standard analytical, probabilistic and statistical tools particularly natural, thanks to the underlying geometry. Moreover, when the required gaps in regularity and weight are satisfied, the relevant embeddings are Hilbert-Schmidt. This property plays an important role in the probabilistic analysis of measure-valued processes, notably in compactness, tightness, and regularity arguments.

Although the Hilbert-Schmidt character of such embeddings is often used in the literature, detailed proofs are rarely included in applications. Such results are typically found in the specialized literature on function spaces, where they are often formulated for substantially more general classes of weights than those considered here. For the reader's convenience, and to make the precise assumptions required in our polynomially weighted setting explicit, we therefore provide a detailed proof of the embedding result used in this paper.

To be precise, weighted function spaces arise naturally in a wide variety of analytical settings, and a well-developed theory is available for several classes of weights. In particular, weighted Sobolev, Bessel-potential, Besov, and Triebel-Lizorkin spaces associated with Muckenhoupt weights have been extensively studied, and we refer, among others, to the works by M.Meyries \& M.Veraar~\cite{meyries-veraar-2012,meyries-veraar-2014,meyries-veraar-2015} and to the references therein. However, the polynomially decaying weights considered here do not, for the range of exponents relevant to our analysis, fall within the Muckenhoupt framework. Indeed, sufficiently strong polynomial decay at infinity prevents these weights from belonging to the corresponding Muckenhoupt classes.

A natural approach would be to rely on a Bessel-potential characterization of the weighted Sobolev spaces. The next result shows that the simultaneous gain of regularity and spatial decay gives rise to a Hilbert-Schmidt operator on
$L^2(\R^d)$. The argument is standard and relies on the kernel representation
of Bessel potentials, see e.g. the book by L.Grafakos~\cite[Definition 1.2.4]{grafakos-2014} for a definition. We highlight that for every $a \in \R_+$, there holds $1+|x|^{2a} \sim (1+|x|^2)^a$, so there is no difference in considering one weight or the other. For the sake of simplicity, we shall denote throughout this section $\rho_a(x):= (1+|x|^2)^{-a/2}$. For simplicity of exposition, we write the proof with the equivalent polynomial weight $\rho_a$. This gives an equivalent norm to the one used in the main text.

\begin{prop}
Let $\ell > d/2$ and $b > d/2$. Then the operator
\[
	T:=\rho_{b}(I-\Delta)^{-\ell/2}
\]
is Hilbert-Schmidt on $L^2(\mathbb R^d)$.
\end{prop}

\begin{proof}
Let $G_\ell$ be the Bessel kernel associated with $(I-\Delta)^{-\ell/2}$, namely
\[
	(I-\Delta)^{-\ell/2}f = G_\ell*f.
\]
Thus $T$ is the integral operator with kernel
\[
	K(x,y) = \rho_{b}(x)G_\ell(x-y).
\]
An integral operator with kernel in $L^2(\R^d \times \R^d)$ is Hilbert-Schmidt (see e.g. the book by M.Reed \& B.Simon~\cite[Theorem VI.23]{reed-simon-1972}), and its Hilbert-Schmidt norm is the $L^2$ norm of its kernel. Namely,
\[
	\|T\|_{\mathrm{H.S.}}^2 = \iint_{\mathbb R^d\times\mathbb R^d} |K(x,y)|^2 \dx\,\di y.
\]
Using the form of $K$, Fubini theorem, and the change of variables $z=x-y$, we obtain

\begin{align*}
	\|T\|_{{\rm H.S.}}^2 &= \iint_{\R^d \times \R^d}\rho_{2b}(x)|G_\ell(x-y)|^2 \dx\,\di y \\
	&= \bigg(\int_{\R^d} \rho_{2b}(x)\dx\bigg)\bigg(\int_{\R^d}|G_\ell(z)|^2\,\di z\bigg).
\end{align*}
The first integral is finite because $b > d/2$. The second one is finite since $\widehat{G_\ell}(\xi)=(1+|\xi|^2)^{-\ell/2}$, and therefore, by Plancherel theorem
\[
	\|G_\ell\|_{L^2}^2 = \int_{\R^d}(1+|\xi|^2)^{-\ell}\,\di\xi <+\infty
\]
if and only if $\ell>d/2$. Hence $K\in L^2(\R^d \times \R^d)$, and $T$ is Hilbert-Schmidt.
\end{proof}

The preceding proposition suggests that the desired Hilbert-Schmidt embedding should result from the combination of two independent effects, namely a gain of regularity, represented by the Bessel potential $(I-\Delta)^{-\ell/2}$, and an additional polynomial decay at infinity, represented by the multiplication operator $\rho_{b}$. However, directly identifying the canonical embedding between the weighted Sobolev spaces considered here with such an operator would require an appropriate Bessel-potential characterization of these spaces. While such characterizations are classical in the unweighted setting, and are also available within several classes of weighted function spaces, their direct use for the weights and the definition adopted here is not immediate. We refer e.g. to the book by T.Hyt{\"o}nen, J.vanNeerven, M.Veraar \& L.Weis~\cite[Section 14.7]{hytonen-vanneerven-veraar-weis-2023} for the general theory of Bessel-potential spaces and related Littlewood-Paley characterizations.

Fortunately, polynomially weighted function spaces have also been studied outside the Muckenhoupt framework. In particular, a theory of weighted Besov and Triebel-Lizorkin spaces associated with nondegenerate polynomial weights has been developed in the works by A.Caetano~\cite{caetano-1998}, by T.K{\"u}hn, H.-G.Leopold, W.Sickel \& L.Skrzypczak~\cite{kuhn-leopold-sickel-Skrzypczak-2006,kuhn-leopold-sickel-Skrzypczak-2006-bis,kuhn-leopold-sickel-Skrzypczak-2007} or also by L.Skrzypczak~\cite{skrzypczak-2005}. Of particular relevance here are the precise estimates for the approximation numbers of embeddings between polynomially weighted function spaces obtained in the work by L.Skrzypczak~\cite{skrzypczak-2005}, see also the correction by L.Skrzypczak \& J.Vyb{\'{\i}}ral~\cite{skrzypczak-vybiral-2009}. To apply these results to the weighted Sobolev spaces used in the present work, we first relate our definition to the multiplicatively weighted formulation appearing in that theory. The following equivalence of norms provides the required link.

\begin{lemma}\label{lem:weighted-sobolev-equivalent-norm}
Let $J \in \N$ and $a \geq 0$. Then there exist constants $c,C>0$, depending only on $J$, $a$, and $d$, such that
\[
    c\|\rho_af\|_{H^J} \leq \|f\|_{\mc H^{J,a}} \leq C\|\rho_af\|_{H^J}.
\]
\end{lemma}

\begin{proof}  Let us first point out that for every multi-index $\gamma$, there exists a constant $C_{a,\gamma}>0$ such that for every $x \in \R^d$
\[
    |\partial^\gamma\rho_a(x)| \leq C_{a,\gamma}\rho_a(x).
\]
This follows by induction from
\[
    \partial_j\rho_a(x) = -a\,\frac{x_j}{1+|x|^2}\rho_a(x)
\]
and from the boundedness of the derivatives of $x_j/(1+|x|^2)$. Applying Leibniz rule, we get that for every $|\alpha| \leq J$,
\[
    \partial^\alpha(\rho_a f) =\sum_{\beta \leq \alpha}{\alpha\choose\beta}(\partial^{\alpha-\beta}\rho_a)\partial^\beta f,
\]
yielding directly
\[
    \|\partial^\alpha(\rho_a f)\|_{L^2} \leq C\sum_{\beta \leq \alpha}\|\rho_a\partial^\beta f\|_{L^2}.
\]
Summing such inequalities over $|\alpha| \leq J$ yields
\[
    \|\rho_a f\|_{H^J} \leq C\bigg(\sum_{|\beta| \leq J} \|\rho_a\partial^\beta f\|_{L^2}^2\bigg)^{1/2},
\]

leading to the first side of the wished inequalit. Conversely, rewriting Leibniz formula as
\[
    \rho_a\partial^\alpha f = \partial^\alpha(\rho_a f) - \sum_{\substack{\beta\leq\alpha\\ \beta\neq\alpha}}{\alpha\choose\beta}(\partial^{\alpha-\beta}\rho_a)\partial^\beta f,
\]
and using once again $|\partial^{\alpha-\beta}\rho_a| \leq C_{a,\alpha,\beta}\rho_a$, we obtain
\[
    \|\rho_a\partial^\alpha f\|_{L^2} \leq \|\partial^\alpha(\rho_a f)\|_{L^2} + C \sum_{|\beta| < |\alpha|}\|\rho_a\partial^\beta f\|_{L^2}.
\]
Finally, an induction on $|\alpha|$ gives
\[
    \|\rho_a\partial^\alpha f\|_{L^2} \leq C\|\rho_a f\|_{H^J},\; \text{for }\, |\alpha| \leq J.
\]
Summing over $|\alpha| \leq J$ proves the reverse estimate and hence the equivalence of norms. The argument is first carried out for $f\in C_c^\infty(\R^d)$ and extends to $H^{J,a}(\R^d)$ by density.
\end{proof}

The preceding lemma shows that, up to equivalent Hilbert norms, the weighted Sobolev spaces considered here can be identified with Sobolev spaces defined through multiplication by a nondegenerate polynomial weight. This allows us to reduce the desired embedding to an embedding between polynomially weighted Triebel-Lizorkin spaces of the type studied by L.Skrzypczak~\cite{skrzypczak-2005}. We may therefore combine this identification with the asymptotic estimates for approximation numbers established therein to obtain the Hilbert-Schmidt property of the canonical embeddings.

\begin{prop}\label{prop:HS-weighted-embedding}
Let $j,\ell \in \N$, $a,b \geq 0$, and assume $\ell, b > d/2$. Then the canonical embedding
\[
    \mc H^{j+\ell,a}(\R^d) \hookrightarrow \mc H^{j,a+b}(\R^d)
\]
is Hilbert-Schmidt.
\end{prop}

\begin{proof}  Let us denote the canonical embedding $I: \mc H^{j+\ell,a}(\R^d) \to \mc H^{j,a+b}(\R^d)$. By Lemma~\ref{lem:weighted-sobolev-equivalent-norm}, the source and target norms are respectively equivalent to $\|\rho_af\|_{H^{j+\ell}}$ and $\|\rho_{a+b}f\|_{H^j}$. We also highlight that $w_b(x) := \rho_{-b}(x) = (1+|x|^2)^{b/2}$, which is precisely the polynomial weight used in the work by L.Skrzypczak~\cite{skrzypczak-2005}. Thus, we define $U : \mc H^{j,a+b}(\R^d) \to H^j(\R^d)$ such that
\[
    U f := \rho_{a+b}f.
\]
Also, by Lemma~\ref{lem:weighted-sobolev-equivalent-norm}, $U$ is bounded and one-to-one. Moreover, for every $g \in H^j(\R^d)$, the function $f := \rho_{-a-b}g$ belongs to $\mc H^{j,a+b}(\R^d)$ and satisfies $Uf = g$, and therefore $\rho_af = \rho_{-b}g$. Hence $U$ is a bounded isomorphism with bounded inverse. Thus, $\|f\|_{\mc H^{j+\ell,a}}$ and $\|\rho_{-b}g\|_{H^{j+\ell}}$ are equivalent. Following the notation of L.Skrzypczak~\cite{skrzypczak-2005}, we define the weighted Triebel-Lizorkin space by
\[
    F^{j+\ell}_{2,2}(\mathbb R^d,\rho_{-b}) := \bigg\{ g \in \mc S'(\R^d)\, \mid\, \rho_{-b}g \in F^{j+\ell}_{2,2}(\R^d) \bigg\},
\]
with norm $\|g\|_{F^{j+\ell}_{2,2}(\rho_{-b})} := \|\rho_{-b}g\|_{F^{j+\ell}_{2,2}}$. Since $F^s_{2,2}(\R^d) = H^s(\R^d)$ with equivalent norms, the preceding identity shows that $U$ also gives an isomorphism, up to equivalent Hilbert norms, from $\mc H^{j+\ell,a}(\R^d)$ onto $F^{j+\ell}_{2,2}(\R^d,\rho_{-b})$.  Hence the embedding $I$ is conjugate, through bounded isomorphisms with bounded inverses, to the canonical embedding
\[
    \widetilde I: F^{j+\ell}_{2,2}(\R^d,\rho_{-b}) \to F^j_{2,2}(\R^d).
\]
In particular, the Hilbert-Schmidt property of $I$ is equivalent to that of $\widetilde I$. We now apply the result by L.Skrzypczak~\cite[Theorem~17(i)]{skrzypczak-2005}. In the notation of that result, we take $p_0 = p_1=2$, $ q_0 = q_1 =2$, $s_0 = j+\ell$, $s_1 = j$, and $\alpha = b$. The parameter denoted there by $\delta$ is
\[
    \delta = s_0-s_1 - d\bigg(\frac{1}{p_0} - \frac{1}{p_1}\bigg) = \ell.
\]
Now since $p_0 = p_1 = 2$ and $b \neq \ell$, the result by L.Skrzypczak~\cite[Theorem~17(i)]{skrzypczak-2005} yields
\[
    a_k(\widetilde I) \sim k^{- {\rm min}\,(b,\ell)/d},\; \text{as } k\to\infty,
\]
where $a_k(\widetilde I)$ denotes the $k$-th approximation number of $\widetilde I$ (see e.g. the book by A.Pietsch~\cite[Definition 2.3.1]{pietsch-1987}). Because both the source and target spaces are Hilbert spaces, the approximation numbers of the compact operator $\widetilde I$ coincide with its singular values $a_k(\widetilde I)=s_k(\widetilde I)$ (see e.g. the book by A.Pietsch~\cite[Theorem 2.11.9]{pietsch-1987}). Therefore, applying the result in e.g. the book by A.Pietsch~\cite[Proposition 2.11.17]{pietsch-1987},
\[
    \widetilde I \in \mathfrak S_2 \Leftrightarrow \sum_{k=1}^\infty a_k(\widetilde I)^2 < \infty.
\]
Using the above asymptotic estimate we have that,
\[
    \sum_{k=1}^\infty a_k(\widetilde I)^2 \sim \sum_{k=1}^\infty k^{-2\min\{b,\ell\}/d},
\]
and since $b, \ell > d/2$ we have $2{\rm min}\,(b,\ell)/d >1$, and thus
\[
    \sum_{k=1}^\infty a_k(\widetilde I)^2<\infty.
\]
Hence $\widetilde I$ is Hilbert-Schmidt thanks to a well-known result (see e.g. the book by A.Pietsch~\cite[Theorem 1.4.5]{pietsch-1987}).

Finally, the class of Hilbert-Schmidt operators is stable under composition with bounded operators. Since $I$ and $\widetilde I$ are conjugate through bounded isomorphisms with bounded inverses, it follows that $I$ is Hilbert-Schmidt.

It remains to consider the limiting case $b = \ell$. Choose $b_0 \in (d/2,b)$. Since $b_0 \neq \ell$, the preceding argument gives the Hilbert-Schmidt embedding $\mc H^{j+\ell,a}(\R^d) \overset{\mathrm{H.S.}}{\hookrightarrow}  \mc H^{j,a+b_0}(\R^d)$. Moreover, because $b_0 < b$, the canonical embedding $\mc H^{j,a+b_0}(\R^d) \hookrightarrow \mc H^{j,a+b}(\R^d)$ is continuous. Hence the canonical embedding $\mc H^{j+\ell,a}(\R^d) \hookrightarrow \mc H^{j,a+b}(\R^d)$ is Hilbert-Schmidt as the composition of a Hilbert-Schmidt operator with a bounded operator, which ends the proof.
\end{proof}

\vspace{2cm}

{\large \textbf{Acknowledgements}.} Geoffrey Lacour was partially funded by the NumOpTES project (ANR22-CE46-0005) of the French National Research Agency. This work was partially carried out
while Geoffrey Lacour was a postdoctoral researcher at Institut de Mathématiques de Bordeaux, CNRS UMR 5251.

{}

\end{document}